\documentclass{eajam}

\usepackage{charter}
\usepackage[charter]{mathdesign}
\usepackage{amsmath, amssymb,amsthm}
\usepackage{color}
\usepackage{subfigure}
\usepackage{multirow}
\usepackage{comment,soul,xcolor,todonotes}
\usepackage[normalem]{ulem}

\renewcommand\thisnumber{x}
\renewcommand\thisyear {2018}
\renewcommand\thismonth{xxx}
\renewcommand\thisvolume{xx}

\newcommand{\E}{\mathbb{E}}

\newcommand{\TV}{\operatorname{TV}}

\def\revdraft{0}
\if\revdraft 1

\else

\fi

\begin{document}

\markboth{S. Tang and H. Yu}{Total Variation Denoising for Urban Traffic Analysis}
\title{Application of Bounded Total Variation Denoising in\\ Urban Traffic Analysis}


\author[S.~Tang and H.~Yu]{Shanshan Tang\affil{1} and Haijun Yu\affil{2}\comma\corrauth}
\address{\affilnum{1}	School of Mathematical Sciences, University of Chinese Academy of Sciences;
	LSEC, Institute of Computational Mathematics and Scientific/Engineering Computing,
	Academy of Mathematics and Systems Science, Beijing 100190, China\\
\affilnum{2}NCMIS $\&$ LSEC, Institute of Computational Mathematics and Scientific/Engineering Computing,
	Academy of Mathematics and Systems Science, 
	Beijing 100190, China;\\
	School of Mathematical Sciences, University of Chinese Academy of Sciences
}
%
%
\emails{{\tt hyu@lsec.cc.ac.cn	} (H.Yu),{\tt tangshanshan@lsec.cc.ac.cn} (S.Tang)}
%
\begin{abstract}
  While it is believed that denoising is not always
  necessary in many big data applications, we show in this
  paper that denoising is helpful in urban traffic analysis
  by applying the method of bounded total variation
  denoising to the urban road traffic prediction and
  clustering problem.  We propose two easy-to-implement
  methods to estimate the parameter noise strength in the
  denoising algorithm, and apply the denoising algorithm to
  GPS-based traffic data from Beijing taxi system. For the
  traffic prediction problem, numerical experiments show
  that the predicting accuracy is improved by applying the
  proposed bounded total variation denoising algorithm. {For
    the} clustering problem, {we apply} a recently developed
  {parameter-free} clustering analysis method for more than
  one hundred urban road segments in Beijing based on their
  velocity profiles. {Much} better clustering results are
  obtained after denoising.
\end{abstract}

\keywords{Urban traffic prediction, bounded variation denoising,  noise estimate, cluster analysis}
\ams{65M10, 78A48}

\maketitle


\section{Introduction} 
Transportation has become an important part in our everyday
life, which brings many kinds of traffic problems at the
same time. For example, big cities like Beijing have a trend
to be more and more congested as the enormous growing of
private cars. Many researchers have tried to analysis and
alleviate the traffic congestion problem with different
methods, especially using vehicle information from Global
Positioning System (GPS) as more and more {GPS} data are
available.  {GPS data has been used for} exploring the cause
and propagation of traffic
jams~\cite{Liu2011Discovering,Wang2013Visual}, mining the
fastest driving route~\cite{Yuan2010T}, trying clustering
methods for trajectories to characterize patterns of
traffic~\cite{Kim2015Spatial,Palma2008A}, estimating the
volume of citywide transportation~\cite{Zhan2017Citywide},
predicting the taxi destinations~\cite{Br2015Artificial} and
forecasting the
velocity~\cite{Sadek2004Multi,Wanli2011Real}. All of these
works are based on trajectories and have utilized temporal,
spatial or spatio-temporal properties. In this paper, we are
interested in velocity prediction. As we all know if the
speed is much lower than the average velocity {of vehicles
  on} a road for a relatively long time, then there is a
high possibility that congestion is formed. So {accurate
  traffic velocity forecast is the base} for routine
planning, congestion control and analysis, etc.. However,
such a traffic forecast is not an easy task in practice as a
result of both the limit of methods and data volume and
{quality}. Urban traffic is a quite complex system, leading
that data collected contain noise. Noises may come from two
aspects. On the one hand, due to the number of cars
appearing on an urban road in a short time interval is very
limited, one needs to consider the average number of cars
and the average velocity on a given road as stochastic
processes. On the other hand, measurement error and
processing error(e.g. the GPS mapping error) may be
introduced in the data collection and post-processing
procedures. Correspondingly, there are two ways to deal with
noisy data. One is to treat noise as real, and develop new
methods to improve accuracy(see, e.g.,
\cite{Sadek2004Multi,Wanli2011Real,Br2015Artificial,Kim2015Spatial,Palma2008A}). The
other is to first reduce noise in the collected
data~\cite{Zhu2013A}, and then use the denoised data to
forecast on the base of existing
methods~\cite{Xiao2003Fuzzy,Zheng2016Traffic}.

In this paper, we prefer to the latter way mentioned
above. The reasons are as follows. First, data can be used
in urban traffic analysis without expensive cost are often
{sparse comparing to the complexity of the problem}. So we
need to make the limited data effective rather than using
the noisy data directly.  Second, although some existing methods have some tolerance with noise,
like deep neural
network~\cite{Yarotsky2017Error,Liang2017Why}. The real
performance of the network depends on many factors and we
believe these methods will perform better with
proper-denoised data under the same condition.  {In
  fact, some neural network methods are not robust to noise
  \cite{tang2010deep}, which causes overfitting.  Many
  methods are proposed to prevent overfitting, like early
  stopping, regularization, soft weight
  sharing\cite{nowlan1992simplifying}, denoising
  autoencoders\cite{vincent2008extracting,
    vincent2010stacked} and dropout\cite{maaten2013learning,
    srivastava2014dropout}.}  Third, by using dedicated {and
  robust} denoising algorithms, shallow neural networks {and
  other simple models} using less data can give fairly good
urban traffic prediction, which is favorable in real-time
and online applications comparing to deep neural networks
due to {their} efficiency.

There exist many useful methods developed for image and
video denoising, such as \cite{Rudin1992Nonlinear,
  You.Kaveh2000, M2005Fast, Bao2005Image,
  Baraniuk2007Compressive, Yin.etal2008, Dabov2008Video,
  Cai.etal2009, Wu.Tai2010a, Liu2010A, Ji2010Robust,
  Pang.etal2011, Shi.etal2012, Zhang.etal2017b,
  Yang.etal2018b, wei2017learning,
    sciacchitano2017total,jiang2018A}. Some of these
methods have been applied to traffic analysis, e.g. the
wavelet method~\cite{Xiao2003Fuzzy,Zheng2011Applications}
and the compressive sensing
method~\cite{Zhu2013A,Xu2015The,Zheng2016Traffic}. These
methods work fairly well in some denoising applications, but
they are based on {particular} assumptions on noise and
specific basis functions. We want to {use} a method works
well with less assumptions on noise and basis functions
and the parameters of which are
easy to obtain.  Thus we apply bounded
total variation (BV) denoising
method~\cite{Rudin1992Nonlinear} to urban traffic velocity
data. It {has edge preserving
  property}~\cite{Strong2003Edge}, which is consistent with
the sudden jump of traffic speed. Our goal is to find the
reconstructed velocity in the function
space of bounded total variation and estimate the noise
strength during the process of denoising with less
assumptions on the type of noise. To this end, we
{propose} two efficient methods to estimate
the noise strength, which almost do not request any
assumptions on noise.

{The rest of the paper is organized as follows. In section 2}, we first review the method of BV denoising, then show how to apply the BV denoising method to traffic data by proposing two different ways to estimate the noise strength. {In section
 3,} we show the effectiveness of denoising {by applying the proposed methods to Beijing urban road {velocity data} derived from taxi GPS system.} {We end the paper with some concluding remarks in section 4.}

\section{The BV denoising algorithm for traffic velocity data}
In this section, we first {give a brief introduction} to the
method of BV denoising, and then show how to apply it to
traffic velocity data.

\subsection{BV denoising model}

Let $u_0(x)$, $x\in\Omega\subseteq \textbf{R}^d$ be the
original observed velocity vector and $u(x)$ be the denoised
velocity vector. Then $u_0(x)=u(x)+\xi(x)$, where $\xi(x)$
is the noise whose mean is 0. We first consider the BV
denoising method for continuous case and then discuss the
corresponding discrete case. {For the continuous case}, the
BV denoising method is to find the solution {$u$ by solving}
the following optimization problem~\cite{Rudin1992Nonlinear}
\begin{equation}
\label{eq:BVconti}
\min_{u(x)} \TV(u)
=
\min_{u(x)}\int_\Omega|\triangledown u(x)|dx
\end{equation} 
such that
\begin{equation}\label{eq:Constrconti}
\int_{\Omega}\!u(x)dx=\int_{\Omega}\!u_0(x)dx,
\qquad \frac12\!\int_{\Omega}\!\big(u(x)-u_0(x)\big)^2dx=\sigma^2,
\end{equation}
{where $\sigma$ is the noise strength parameter need to be specified.}
When we apply the BV denoising method to reduce the noise in
road velocity time series, $x$ is one dimensional, so we
only consider the case that $d=1$ below.  The first order
condition of the optimization problem satisfies
\eqref{eq:Constrconti} and following Euler-Lagrange
equation:
\begin{equation*}
\frac{d}{dx}\Big(\frac{u_x}{|u_x|}\Big)-\lambda_1-\lambda_2(u-u_0)=0, 
\qquad \frac{\partial u}{\partial n}\Bigr|_{\partial \Omega} =0, 
\end{equation*}
where $u_x=\frac{\partial u}{\partial x}$, $\lambda_1$ and
$\lambda_2$ are the corresponding Lagrange multipliers.  To
obtain the solution of the optimization problem
\eqref{eq:BVconti}, one can solve the following equations
corresponding to the gradient flow of the Lagrange function
with respect to $u$:
\begin{align}
& u_t=\frac{d}{dx}\Big(\frac{u_x}{|u_x|_{\epsilon}}\Big)-\lambda_2(u-u_0),\label{eq:LfunGradu}\\
& \frac{\partial u}{\partial n}\Bigr|_{\partial \Omega}=0,\label{eq:NeumanBC2}\\
& u(x,0) =u_0(x)+\sigma\phi, \label{eq:uIC}
\end{align}
where $\phi$ in \eqref{eq:uIC} satisfies $\E(\phi)=0$,
$\frac{1}{2}\int_{\Omega}\phi^2dx=1$ such that $u(x,0)$
satisfies the constraint \eqref{eq:Constrconti}, and
$|x|_{\epsilon}=|x|+\epsilon$ {with $\epsilon$ is a
  relatively small number} to avoid singularity and obtain
better numerical stability. In addition, when initial
condition \eqref{eq:uIC} is used, we can assure
$\lambda_1=0$ all the time. So we discard $\lambda_1$ and
only keep $\lambda_2${, which is rewritten as $\lambda$ for
  simplicity.}  {To} calculate $\lambda$, {we first}
multiply $u(x)-u_0(x)$ on both sides of
\eqref{eq:LfunGradu}, {then} integrate {the result} on
$\Omega$.  By the Neumann boundary condition and
\eqref{eq:Constrconti}, we obtain:
\begin{equation}
\lambda=\frac{1}{2\sigma^2}\int_{\Omega}\frac{u_x}{|u_x|_{\epsilon}}\cdot((u_0)_x-u_x)dx.
\end{equation}
$\lambda$ changes as $u_x$ changes with respect to $t$, and
converges when $t$ {goes to}
infinity~\cite{Rudin1992Nonlinear,Rosen1961The}.

{The} steady state solution of
\eqref{eq:LfunGradu}-\eqref{eq:NeumanBC2} {represents} the
clean data denoised from the observed data $u_0(x)$. Next,
we present numerical schemes to solve equations
\eqref{eq:LfunGradu}-\eqref{eq:NeumanBC2} in discrete case.

\subsection{The projected gradient descent algorithm}

The data we used are average traffic velocities on urban
roads derived from the taxi GPS records. If we have a
velocity record every other 5 minutes each day, then we have
288 records for each road every day. However, the velocity
records are often less than 288 for most roads because of
the limit of taxi numbers and recording
equipments. Therefore, we make nearest interpolation {for
  missing data} such that each road has 288 velocity records
each day. In this way, we set $N=288$, and let $x$ direction
correspond to direction of $N$ time slices. Define
\begin{equation}\label{eq:Def-sigma}
\frac{1}{2}\sum_{i=1}^{N}(u(x_i)-u_0(x_i))^2\cdot h=\sigma_d^2,
\end{equation}
where $h$ is the duration of time intervals{,} in which the
traffic velocity on road segments are {averaged}. For the
data we used, $h$ is 5 minutes. {$x_i=[(i-1)h, ih)$},
$i=1,2,\ldots,N$ are the time intervals in one day,
$u_0(x_i), u(x_i)$ are the observed and denoised velocity
for time slice $i$, $\sigma_d$ corresponds to the noise
strength in discrete case. While $\sigma$ and $\sigma_d$ are
different, we do not distinguish them in the following for
convenience.

We use projected gradient descent(PGD)
method~\cite{Rosen1961The} to find the numerical solution of
equations \eqref{eq:LfunGradu}-\eqref{eq:NeumanBC2} because
it is relatively {straightforward} and {good enough for our
  goal in this study}. It is guaranteed that the {PGD}
method converges~\cite{Antonin2009An}.  {We note that there
  exist several more advanced numerical algorithms for
  complicated denoise problems, such as the Bregman
  iterative algorithms \cite{Yin.etal2008, Cai.etal2009,
    Shi.etal2012} and augmented Lagrangian method
  \cite{Wu.Tai2010a, Wu.etal2011}, etc.}

Denote $u_i^n=u(x_i, t^n)$, where $n$ is the iteration
index, {$\Delta t^n =t^{n+1}-t^n$} is the step size of the
$n$-th iteration. The detailed projected gradient descent
bounded variation denoising (PGDBV) algorithm is described
as follows. 

{\bf Algorithm PGDBV:} The algorithm takes the observed
velocity series:
$\left\{u_0(x_i),\ i=1,\ldots,N\, \right\}$; the variance of
the noise: $\sigma$; the maximum iteration number:
$N_{iter}$; the relative tolerance of gradient decay:
$\delta$ as {\it input}.  The denoised velocity series:
$\left\{ u_i,\ i=1,\ldots,N \right\}$, and the total
variation of denoised data: $V$ are calculated by performing
the following steps.
\begin{enumerate}
\item Initialize velocity series:
  $u^0_i=u_0(x_i), i=1,\ldots,N$, and calculate the initial
  discrete total variation:
  \begin{equation}\label{eq:BVdisc}
	V^n = \TV(u^n) =\sum_{i=1}^{N-1} |u^n_{i+1} - u^n_i|,
	\end{equation}
    with $n=0$. {If $V^n=0${,} then set
      $u_i=u_i^0, i=1,\ldots, N$, $V=V^0$ and return.}
  \item For $n$ from $0$ to $N_{iter}-1$, do the following
    iterations
	\begin{enumerate}
    \item Update the Lagrange multiplier $\lambda^n$ as
      \begin{equation*}
		\lambda^n=\frac{{h}}{2\sigma^2}\sum_{i=1}^{N-1}\frac{\triangle^+u_i^n}{|\triangle^+u_i^n|_{\epsilon}}(\triangle^+(u_0)_i-\triangle^+u_i^n),
      \end{equation*}
      {where $\Delta^+u^n_i := u^n_{i+1} - u^n_i$.}
    \item Calculate the gradient of the $i$-th segment as
      \begin{align*}
		g_i^n&=-\Bigl( \frac{1}{h}\bigl( \frac{\triangle^+u_i^n}{|\triangle^+u_i^n|_{\epsilon}} - \frac{\triangle^+u_{i-1}^n}{|\triangle^+u_{i-1}^n|_{\epsilon}}\bigr) 
               - \lambda^n(u_i^n-u_0(x_i)) \Bigr), \quad \mbox{for}\ i=1\ldots N,
      \end{align*}
      where $\triangle^+u_0^n=0$, and
      $\triangle^+u_{{N}}^n=0$.
    \item Obtain step size $\triangle t^n$ with line search;
    \item Update the variable
      $u^{n+1}=u^n-\triangle t^ng^n$;
    \item Calculate the total variation $V^{n+1}$ using
      \eqref{eq:BVdisc} with $n$ replaced by $n+1$;
    \item If $\| g^n \|_{\infty}/V^0 \le \delta$,
      then set $u_i=u_i^{n+1},i=1,\ldots,N, \;V=V^{n+1}$ and
      return.
	\end{enumerate}
    \item Set $u_i=u_i^{N_{iter}},i=1,\ldots,N, \;V=V^{N_{iter}}$ and return.
\end{enumerate}	

\subsection{Estimate of the noise strength}
In Algorithm PGDBV, {input $\sigma$} is a characterization
of the noise strength. But we do not know its value a
priori. Now we propose two different ways to estimate
{$\sigma$}.

\textbf{Method 1}: Multi-resolution noise estimate.

In each day, every road has $N=288$ velocity records
corresponding to $N$ time slices. We want to estimate the
noise strength in three different resolutions: $N_1=N$,
$N_2=\frac{N}{2}$, $N_3=\frac{N}{4}$. Denote
$[N]= \{1,2,\cdots,N\}$. Let {
\begin{equation}\label{eq:vdef}
v_i^0=u_0(x_i),\; i\in[N_1]; \qquad
v_i^{j}=\frac{v^{j-1}_{2i}+v^{j-1}_{2i-1}}{2},\; i\in[N_{j+1}],\; j=1,2.
\end{equation}
}
Then the variation in three resolutions are:
	\begin{equation*}{
	 V_{j+1}=\sum_{i=1}^{N_{j+1}-1}|v^j_{i+1}-v^j_i|^2\cdot\frac{1}{2^j h},
	 \quad j=0,1,2.}
	\end{equation*}

    {The estimate of $\sigma$ together with an error bound
      is given in the following Theorem. The proof can be
      found in Appendix.}
\begin{theorem}\label{thm: Thm1}
  Assume $u_0(x_i)=u_i+\xi_i, i\in[N]$, where $\xi_i$ is
  independent with $\xi_j$ for $i\neq j$, $\E(\xi_i)=0$,
  $\E(\xi_i^2)=\frac{2\sigma^2}{Nh}$, $i\in[N]$, then
	\begin{equation}\label{eq:SigmaEstimate1}
	\hat{\sigma}^2=h^2\frac{(\frac{119}{16}-\frac{27}{4N})V_1+(\frac{9}{4N}-\frac{49}{16})V_2+(\frac{9}{2N}-\frac{35}{8})V_3}{\frac{3577}{128}+\frac{189}{8N^2}-\frac{819}{16N}},
	\end{equation}
	{is a consistent estimate of $\sigma$}, in the sense that
	\begin{equation}\label{eq:varEstimate1}
	\E｛(\hat{\sigma}^2)｝-\sigma^2=
	h^2\frac{(\frac{49}{16}-\frac{9}{4N})(V_1^c-V_2^c)+(\frac{35}{8}-\frac{9}{2N})(V_1^c-V_3^c)}{\frac{3577}{128}+\frac{189}{8N^2}-\frac{819}{16N}},
	\end{equation}
	where
	\begin{equation*}
	V_{j+1}^c=\sum_{i=1}^{N_{j+1}-1}\frac{|u^{(j)}_{i+1}-u^{(j)}_i|^2}{2^{j-1}h},\quad\mbox{for}\ j=0,1,2,
	\end{equation*}
	with
	 $u_i^{(0)}=u_{i}, i\in [N_1]$,
	$u_i^{(j)}=\frac{u_{2i-1}^{(j-1)}+u_{2i}^{(j-1)}}{2}, i\in[N_{j+1}], j=1,2.$
\end{theorem}

\begin{remark}
  From the representation of $\E(\hat{\sigma}^2)$ in Theorem
  \ref{thm: Thm1}, it can be seen $\hat{\sigma}$ {depends
    on} {$h$}.
	If $h$ is smaller, then the estimation can be closer to the real unknown $\sigma$.
\end{remark}

We {now use} two examples to {test} the ability of the
method above: 1) $u(x)=\sin(\pi x)$, $x=[-1,1]$,
$h=\frac{2}{N}$; 2) $u(x)=\max(1-|x|,0)$, $x=[-1,1]$,
$h=\frac{2}{N}$.  We add white noise whose standard
deviation is $\sigma$ such that the signal-to-noise ratio is
3. Test results are in Table \ref{tb:1}. From Table
\ref{tb:1}, we found that the numerical results of proposed
method for estimating $\sigma$ are very good. However, some
assumptions on noise are required in this method. Next, we
present a method use less assumptions.

\begin{table} [hbp]
	\caption{Numerical Results of Noise Estimate Using Method 1.}
	\label{tb:1}
	\centering
	\begin{tabular}{l|lll|llll}
		\hline
		& \multicolumn{3}{|c|}{$u(x)=\sin(\pi x)$}  & &
		\multicolumn{3}{c}{$u(x)=\max(1-|x|,0)$} \\
		\hline
		$N_l$ & 288& 144& 72 & 	 & 288& 144& 72 \\
		\hline
		$\sigma$& 0.4941& 0.4446& 0.4385 & 	& 0.4127& 0.4176& 0.3992 \\
		\hline
		$\hat{\sigma}$& 0.4896& 0.4558& 0.4646 && 0.4089& 0.4290& 0.3813 \\
		\hline
		$Bias(N_l)$& 2e-6& 1.69e-5& 1.47e-4 & 	& 4.47e-7& 3.59e-6& 2.89e-5\\
		\hline
		$\hat{\sigma}^2/\sigma^2-1$& -0.0182& 0.0508 & 0.1230 && -0.0182& 0.0554& -0.0878\\
		\hline
	\end{tabular}
\end{table}

\textbf{Method 2}: a balance between $\sigma$ and $TV$.

From the BV {denoising model}, we see that when $\sigma=0$,
$TV(u)=TV(u_0)$. {As $\sigma$ increases, $TV$ decreases, and
  eventually goes to $0$}.  In {this} process, a critical
point exists at which the decreasing speed of $TV$ has a
{significant} slow down .  We regard this point as
$\hat\sigma$, and give a quantitative way to identify
{it}. Choose some values for $\sigma$ in increasing order
first. And for any two successive $\sigma_1$ and $\sigma_2$,
$\sigma_1<\sigma_2$, calculate the variance
$\delta(TV(\sigma_2)\cdot\sigma_2^2)$ as
\begin{equation}
\delta(TV(\sigma_2)\cdot\sigma_2^2))=TV(\sigma_2)\cdot\sigma_2^2-TV(\sigma_1)\cdot\sigma_1^2.
\end{equation}
Plot the result of $\delta(TV(\sigma)\cdot \sigma^2)$
against $\sigma^2$. Then find the $\sigma$ where
$\delta(TV(\sigma)\cdot \sigma^2)$ reaches {the local}
minimum value for the first time, and regard it as the value
for $\hat{\sigma}$.  {We use
  $\delta(TV(\sigma_2)\cdot \sigma_2^2)$ as an index because
  $TV(\sigma_2)\cdot \sigma^2_2$ increases first and then
  decreases with the increase of $\sigma_2$. The first stage
  is mainly due to noise reduction, but in the second stage
  some of the real $TV$ is removed. So we need a good
  threshold. By some numerical tests, we found that the
  location where the increment gets its first local minimum
  is a good choice. }

We have presented two methods to estimate noise strength
${\sigma}$. However, both of them have advantages and
disadvantages. The first method {is easier to apply}, but
the estimation result depends on some related
assumptions. The second method does not need any assumption
on noise, but it is heuristic.  To get a better estimate, we
combine the two methods together. Let $\hat{\sigma}_1$,
$\hat{\sigma}_2$ denote the estimated noise strength with
\textbf{Method 1} and \textbf{2}.  Denote the maximum,
minimum velocity $v_{\max}$, $v_{\min}$ for a road in a
day. Then take $TV_l=\frac{5}{2}(v_{\max}-v_{\min})$ as a
lower bound of $TV$.  If the $TV$ corresponding to
$\min(\hat{\sigma}_1,\hat{\sigma}_2)$ is less than $TV_l$,
then let $\hat{\sigma}$ be the $\sigma$ corresponding to
$TV_l$, otherwise
$\hat\sigma = \min(\hat{\sigma}_1,\hat{\sigma}_2)$.

\section{Experiments}
In this section, we use taxi GPS data of Beijing to
demonstrate the performance of our method. There are 288
time slices in one day, {each is 5-minute long}. If no GPS
record appears in some {slice}, we get no velocity record in
that time slice. In general, if one method performs well in
a randomly chosen domain of Beijing, then it can be extended
to other domains. Therefore, we randomly choose an area from
Beijing and consider the GPS records for roads in it. First,
we choose 24 roads at random from this fixed domain, with
road length not less than 100 meters and daily record number
not less than 150. {The results for more roads are similar,
  the reason we choose 24 roads is to show the results in
  limited page space.}  We use the 24 roads to test the
estimated noise strength with methods proposed in section 2
and combine them together to obtain the best estimate. After
that, the velocity of these 24 roads before and after
denoising on a fixed day are presented to give a visual
effect directly.  In further, we use history matching to
prove that predicting accuracy is improved after denoising
with the best estimated noise strength.  Second, 102 roads
in the chosen domain, whose length is not less than 100
meters and velocity records are not less than 120, are put
together to see the denoising effect on clustering. As a
result, a satisfactory clustering consequence is obtained
after denoising.

\subsection{Estimated best noise strength}

We carry out experiments on the 24 roads for 8 days, and
only present some related results on the first day in this
part {since the results for other days are similar}. Day 1
to 8 correspond to March 4-8, 18-19, 11, 2013, all of which
are weekdays. {Since} there are time slices without velocity
record, we use linear nearest interpolation and matrix
completion methods {\cite{Cand2009Exact,Fazel2001A}} to
complete the missing values. The estimated $\sigma$ for the
24 roads on Day 1 with two different kinds of {treatments
  for missing data} by \textbf{Method 1} are shown in Figure
\ref{Fg:SigmaM1}, which {shows that the differences are
  small}.
\begin{figure*}[!ht]
	\centering
		\includegraphics[width=2.8in]{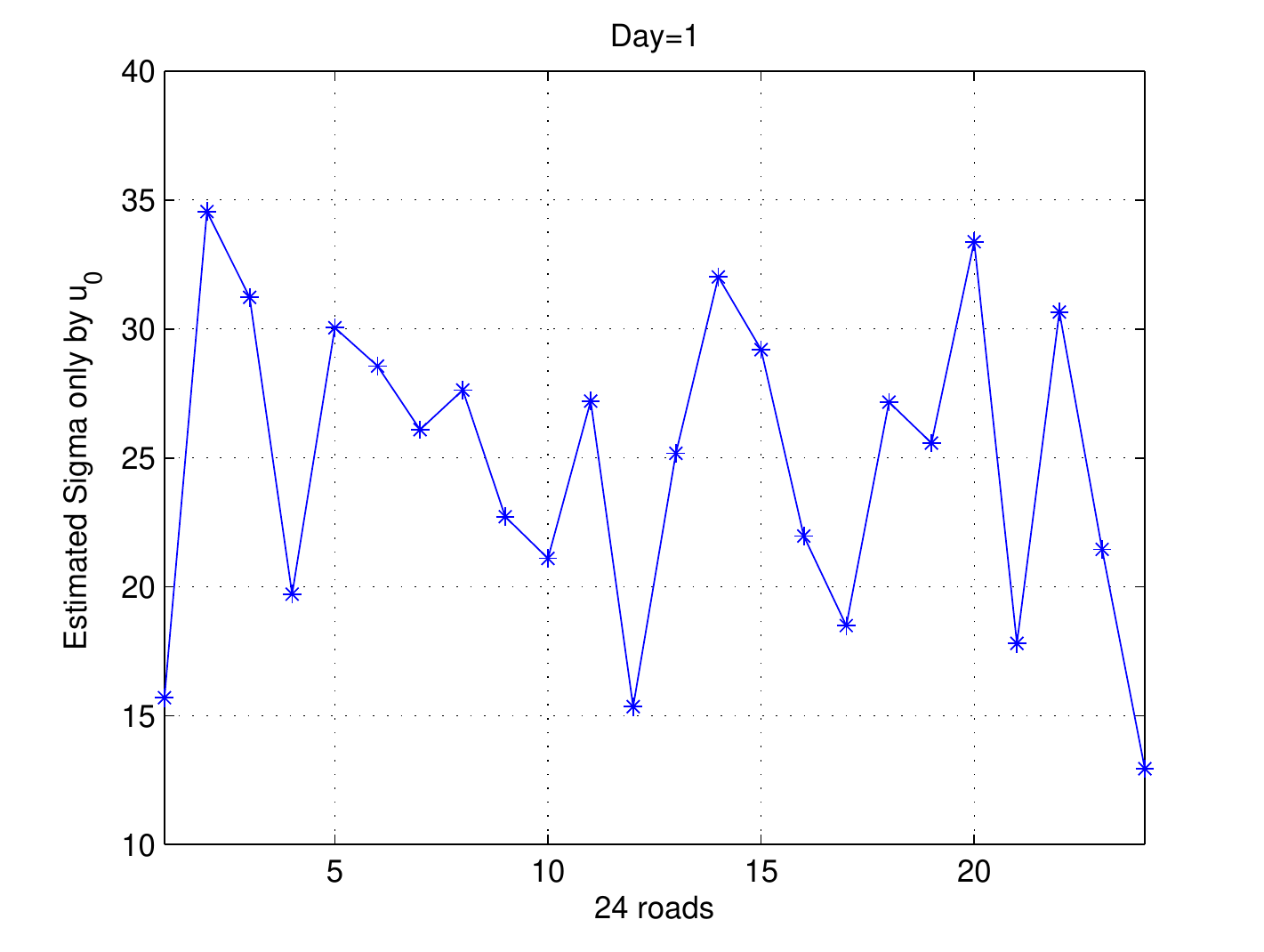}
		\includegraphics[width=2.8in]{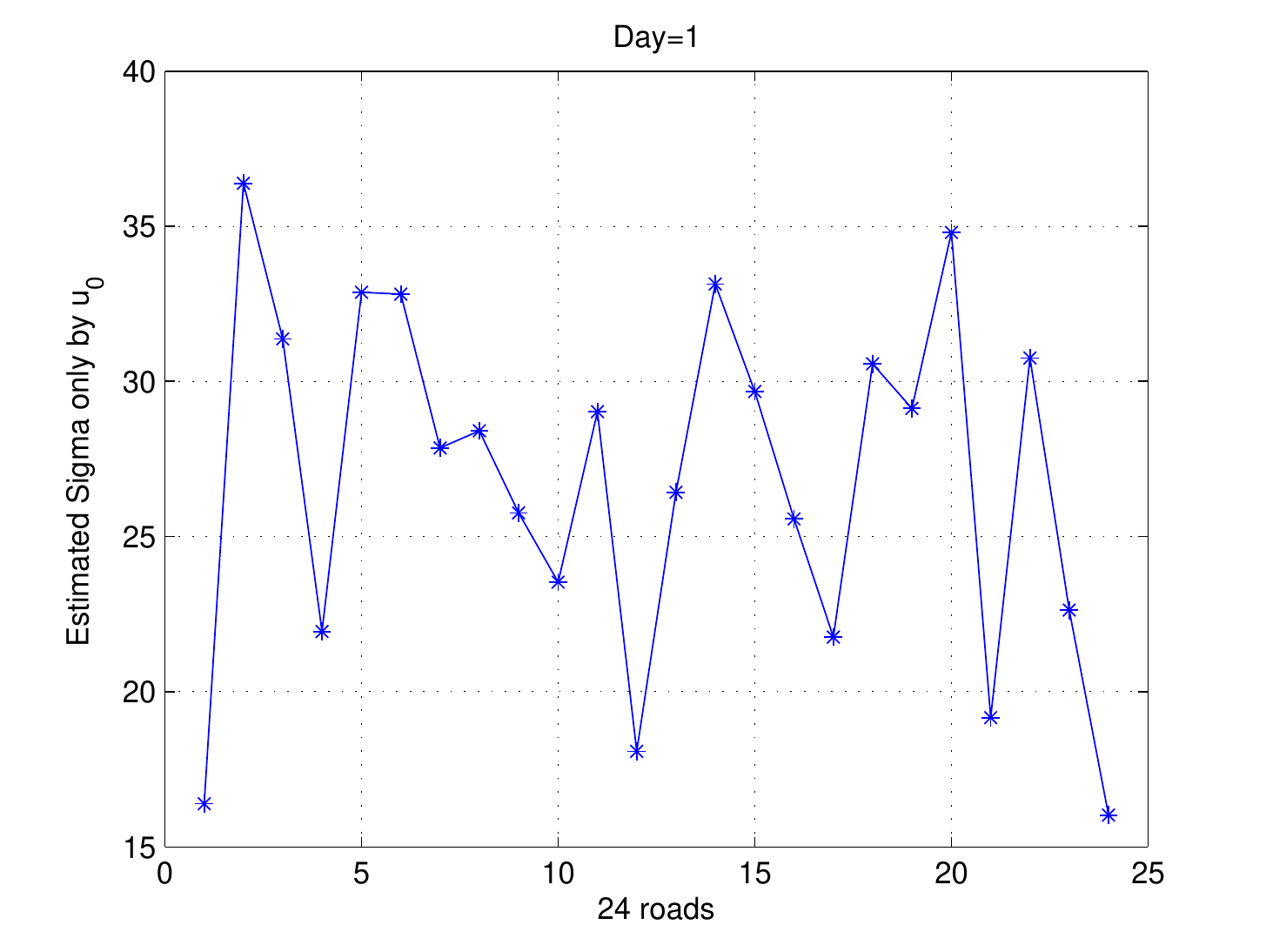}
		\caption{Estimated $\sigma$ by \textbf{Method 1}
          using different methods for missing values (Left:
          Nearest interpolation; Right: Matrix
          completion). }
		\label{Fg:SigmaM1}
\end{figure*}
Then we estimate noise strength with \textbf{Method
  2}. First, Figure \ref{Fg:BV} shows the relationship
between $TV$ and $\sigma^2$, where the horizontal and
vertical axis represent $\sigma^2$ and $TV$
respectively. The values of $\sigma$ are $0,1$ {and every
  other 5 from $5$ to $50$}, totally 12 numbers. In this
figure, red star points correspond to the 12 values of
$\sigma^2$, and blue star points denote the cases for
$\sigma^2=[10^2,20^2,30^2]$. It can be discovered that, for
each road, $TV$ is decreasing with the growth of
$\sigma^2$. The quantitative way to identify $\hat\sigma_2$
in Method 2 is {demonstrated} in Figure \ref{Fg:BVSigma^2},
where the horizontal axis is the value of $\sigma^2$, and
vertical axis is $\delta(TV(\sigma)\cdot\sigma^2)$. For
example, the abscissa of the first point on road 1 is
$\sigma=1$, and its ordinate is $\delta(TV(1)\cdot
1^2)$. $\hat{\sigma}_2^2$ is the location where
$\delta(TV(\sigma)\cdot\sigma^2)$ reaches minimum value for
the first time. For road 1, this minimal point is reached at
$\sigma=20$, which implies $\hat{\sigma}_2=20$. {In Figure
  3, the intersection of the green line with the horizontal
  axis is the best estimated $\sigma^2$.  The results of
  combining the two methods together is also shown in Figure
  \ref{Fg:BV}, where the values of $\hat{\sigma}_1$,
  $\hat{\sigma}_2$ are the intersections of the blue, green
  vertical line and the horizontal axis separately. And blue
  circle points are the best estimates for noise strength.
  It can be seen from Figure 2 that all the best estimates
  for the noise strength lie in
  $(\hat{\sigma}_2-5,\hat{\sigma}_2+5)$ except road 17 and
  18. Therefore, {the criterion we chosen in method 2 is
    appropriate.}

\begin{figure*}[!ht]
	\centering
	\includegraphics[width=\textwidth]{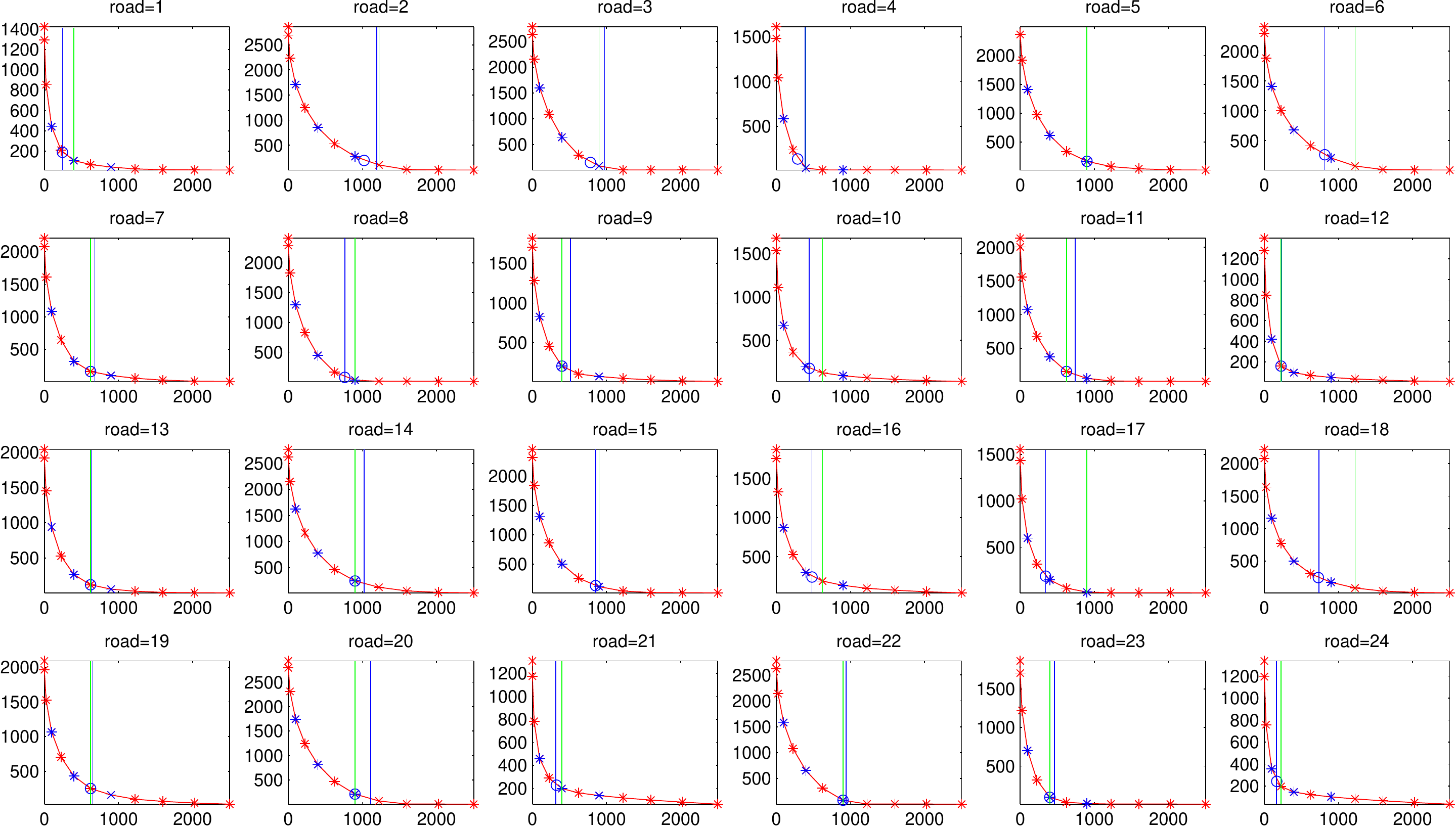}
	\caption{Relationship figure for $TV$ and $\sigma^2$ on Day 1.  }
	\label{Fg:BV}
\end{figure*}
\begin{figure*}[!ht]
	\centering
	\includegraphics[width=\textwidth]{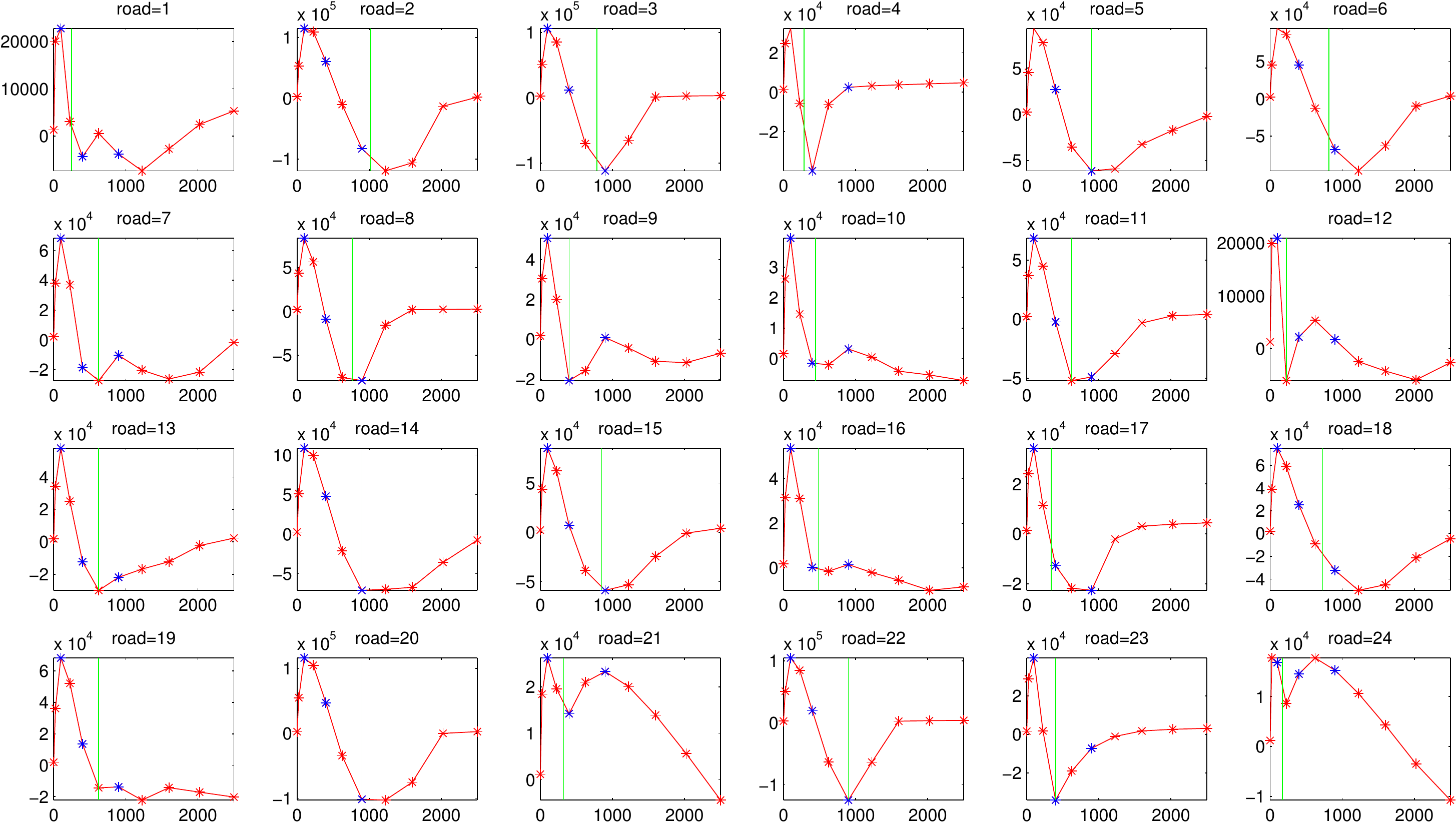}
	\caption{Estimated $\sigma$ for the 24 roads on Day 1 with \textbf{Method 2}. }
	\label{Fg:BVSigma^2}
\end{figure*}

\subsection{Velocity before and after denoising}

We show the velocity before and after denoising on Day 1
{with best chosen $\sigma$ for the 24 roads } in Figure
\ref{fg:V}, where blue and red points represent the
{velocity} before and after denoising
{respectively}. {Nearest interpolation is used to handle
  missing values}.  From Figure \ref{fg:V}, we can say we
have reduced much noise in the velocity
indeed. Specifically, the {velocity} vector is much smoother
after denoising than that of before denoising. And the
velocity after denoising keeps the peak values, such as the
peak in the morning and evening. In addition, it can be
observed {that} when the original velocity has a sudden
jump, the denoised velocity also has a jump almost at the
same time, which is {clearly observed} on road 21 and
24. Namely to say, our denoising method can keep the jump
when the original velocity has a sudden ascending or
descending, which is very important in urban traffic
prediction and analysis. This is consistent with the edge
preserving property of bounded total variation method. On
the other hand, the denoising effect of some roads, like
road 3, 8, 22, are not that obvious as others although the
performance is better than the original velocity. This may
be caused by their location. They are close to
{intersections} and are all assistant roads, which {have
  big} spatial noise, while we only consider temporal noise.
\begin{figure*}[!ht]
	\centering
	\includegraphics[width=\textwidth]{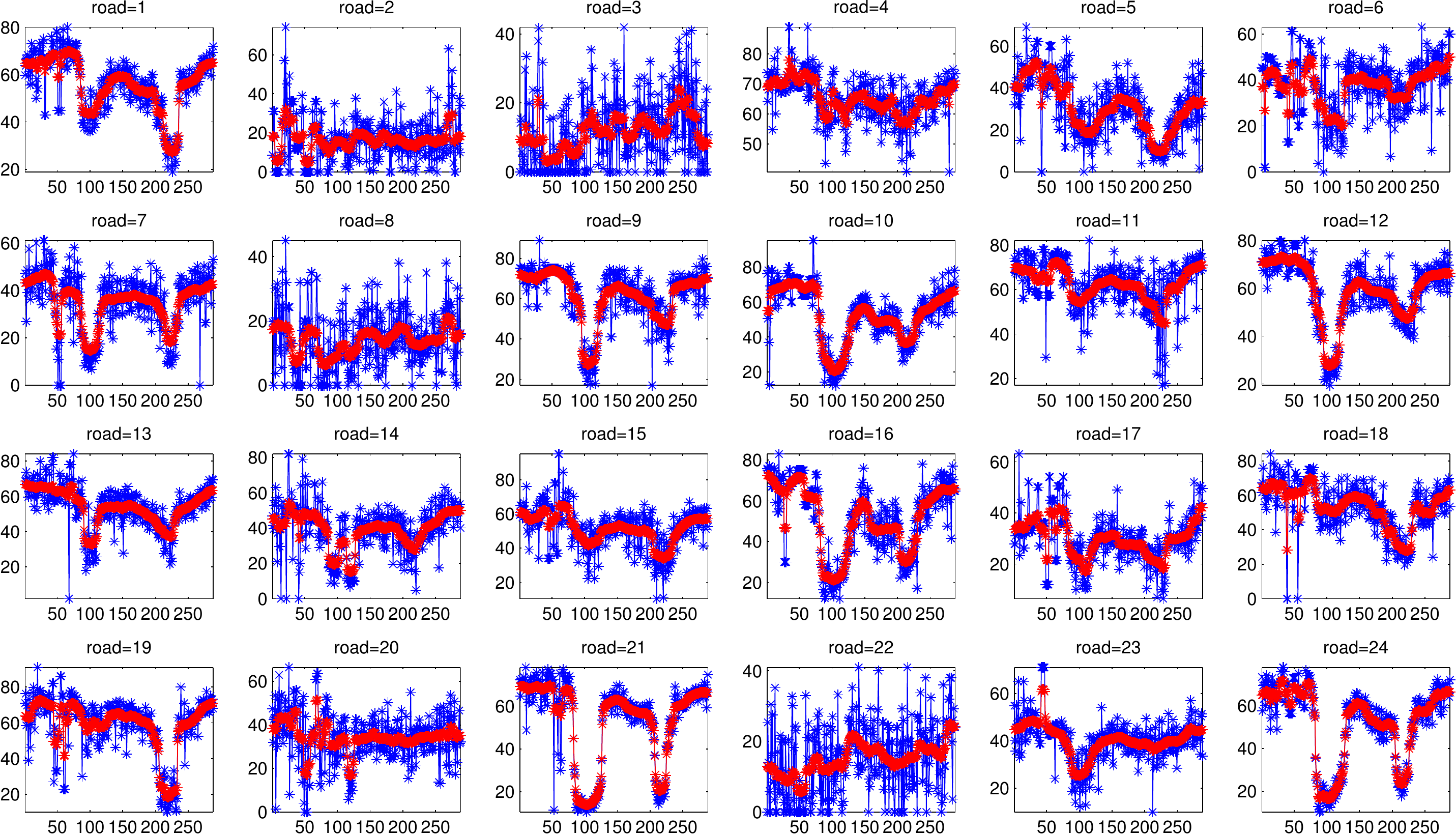}
	\caption{Velocity before and after denoising for the 24
      roads on Day 1. The horizontal axis is 288 time
      slices, and the vertical axis is the velocity value. }
	\label{fg:V}
\end{figure*}

\subsection{Validation by history matching}

We have seen the denoising effect directly above. {Now} we
validate that this denoising method helps in velocity
prediction. We first make a nearest
interpolation 
to assure there are 288 records every day for all the chosen
roads.  Data of the first 7 days are regarded as history
data, and the prediction part is from Day 8. We realize it
based on history matching. Specifically, we break the
history data daily into vector of length 4 in terms of time
series. Take Day 1's data for example, 282 vectors are
obtained. The first one of them is the velocity records from
time slice 1 to 4, and the last one is that of time slice
282 to 285. And the velocity vector after 15 minutes is from
time slice 7 to 288 {respectively}. Therefore, 1974 input
vectors of length 4 and corresponding labels of length 1 are
gained for the history part. And for Day 8's data, the goal
vector {for time slice $K$} is formed by velocity of time
slice $K-3$ to $K$. We make use of a popular clustering
algorithm~\cite{Rodriguez2014Machine} (more details in the
last part of this section) to cluster for the 1974 input
vectors and the goal vector. The distance between two
vectors is measured by $l^2$ norm. After that, we regard the
Gaussian-weighted average of those labels, whose input
vectors are in the same cluster with the goal vector, as the
velocity prediction.

What we {described above} is exactly the prediction process
without BV denoising. For the denoising case, {there is a
  small difference}.  For the history data, Algorithm PGDBV
can be applied daily to obtain the denoised data. However,
for data from Day 8, it should be guaranteed that no future
information is used to denoise and predict, leading that the
denoising part should be treated carefully. Notice that a
boundary condition is needed in the denoising process of
Algorithm PGDBV, thus we {supplement} a boundary value for
time slice $K+1$, and use the velocity vector till $K+1$ to
denoise, and component of the goal vector after denoising is
still from $K-3$ to $K$.

{To get the boundary condition}, we use a very simple neural
network based on data before denoising to make a 5-minute
prediction. Input and output of the network are exactly the
1974 input vectors and corresponding labels before denoising
separately. Input test are the 282 goal vectors before
denoising, and then the network's output test are regarded
as the boundary values to denoise.

We measure the predicting accuracy using relative mean
absolute error (RMAE) and mean absolute percentage error
(MAPE). {They are defined as}:
\begin{equation*}
\text{RMAE}=\frac{\sum_{i=1}^N|u_i-\hat{u}_i|}{\sum_{i=1}^N|u_i|},\quad
\text{MAPE}=\frac{1}{N}\sum_{i=1}^N\frac{|u_i-\hat{u_i}|}{|u_i|}.
\end{equation*}

Notice that when some component of the real speed vector is
very close to 0, MAPE of the corresponding road will be very
large.  Thus we discard component whose real velocity is not
bigger than 1 when calculating a {road's} MAPE. After
obtaining the error of each chosen road, we take the average
of them, which stands for the predicting error on
average. 15-minute predicting results based on nearest
interpolation are listed in Table \ref{tb:15minErr}. The
RMAE reduces at least more than 1.4 percent when choosing
$\hat\sigma$ comparing with no denoising. And the MAPE
reduces 2.6 percent with $\hat\sigma$ for the 24 chosen
roads on average.
\begin{table}[!htbp]
	\caption{Error of 15-minute predicting with nearest interpolation for missing data}
	\label{tb:15minErr}
		\centering
		\begin{tabular}{l|lll|llll}
		    \hline
		    & \multicolumn{3}{|c|}{RMAE}  & &
		    		\multicolumn{3}{c}{MAPE} \\
			\hline 
			$\sigma$& 0 &$\tilde{\sigma}$  &   
			 $\hat\sigma$ & & 0   &$\tilde{\sigma}$  &   
			 			 $\hat\sigma$ \\
			\hline 
			24 roads& & & && & &\\
			average& 0.2641  & 0.2507   & 0.2414&  & 0.3062 & 0.2803     & 0.2802 \\
			\hline
			w/o road & & & && & &\\
			2,3,6,8,22     & 0.1806  & 0.1729   & 0.1660 &
			& 0.2284 & 0.2090 	 & 0.2101\\
			\hline
			w/o road & & & && & &\\
			2,3,6,7,8,22   & 0.1746  & 0.1680   & 0.1605 & & 0.2136 & 0.2023     & 0.1980\\
			\hline
		\end{tabular}
\end{table}

In {this table}, the first line represents values of
$\sigma$ applied in Algorithm PGDBV, where '0' represents
the case without denoising and $\hat\sigma$ means the best
chosen $\sigma$ for each chosen road and each goal vector,
and '$\tilde{\sigma}$' means only denoising for history data
with $\hat\sigma$. The second line means the average
error(RMAE or MAPE) of the 24 chosen roads. The third line
is the result of error for 19 roads except the 2nd, 3rd,
6th, 8th, 22nd roads, and {the 7th road is also excluded for
  the results on} the fourth line. The reason why we skip
road 2, 3, 6, 8, 22 is that they are rural {streets} and the
prediction accuracy is low. In addition, road 7's predicting
error is also high. This may be caused by its location,
which is very near to a bridge. From the predicting errors
shown, it can be seen that the RMAE {with $\hat\sigma$ are
  reduced for all the cases}. This reveals that the
$\hat\sigma$ we choose for the 24 roads are appropriate and
the BV denoising method we {proposed} is {helpful}.

We finish this section {with a short discussion about the
  treatments of missing data}. In {Table \ref{tb:15minErr}},
we use nearest interpolation to recover the missing values
as it is a simple and common way. To compare the
performance, we also test {the} matrix completion
approach. We find the predicting accuracy of matrix
completion approach is improved comparing to the nearest
interpolation approach, {but the improvement is not
  significant}.

\subsection{Clustering of roads using velocity profile}

We randomly choose 102 roads on Day 1 and {cluster them}
according to their velocity profiles before and after
denoising respectively, {to further test the denoising
  algorithm works in a right way.}

When clustering, we regard the velocity of $N$ time slices
for a road as a vector {of dimension} $N$. As a result, we
have 102 vectors in total. We measure the distance of two
vectors with the $l^2$ norm of their differences. The
clustering method is based on finding of density
peaks~\cite{Rodriguez2014Machine}. The most important reason
we choose this clustering method is {that Outliers, which
  can be seen as noise according
  to~\cite{Rodriguez2014Machine}}, can be discovered easily
with this algorithm if exist. Therefore, if there are
outliers in the clustering result with velocity before
denoising and no outliers in the clusters after denoising,
we can say noise is eliminated with our denoising method.
{The framework for validating data cleaning techniques based
  on assumption that better denoising schemes lead to better
  data analysis has been also described in
  \cite{Xiong2006Enhancing}. In further, the impact of noise
  on clustering is shown in \cite{liu2010understanding} and
  reveals noise can cause adverse effects on clustering
  indeed.}

In order to show the process of clustering on the 102 roads
in detail, it is necessary to present the procedure of
clustering algorithm~\cite{Rodriguez2014Machine}.  The steps
of this clustering method combined with our data are
summarized as follows:
\begin{itemize} 
\item{Step 1.} Calculate the local density $\rho_i$ for each
  road $i$. Let $d_{ij}$ be the $l^2$ norm distance between
  road $i$ and $j$, $d_c$ be a scaling constant for
  distance, $I_s$ denote the index set of all the 102
  roads. We use the Gaussian kernel to calculate $\rho_i$,
  which is defined as
  $\rho_i=\sum_{j\in{I_s\backslash
      \{i\}}}\exp{(-\frac{d_{ij}}{d_c})^2}$.
\item{Step 2.} Calculate the distance to other roads with
  higher density for {each} road $i$:
  $\delta_i=\min_{j: \rho_j>\rho_i}(d_{ij})$.
\item{Step 3.} Plot the {\em{decision graph}}, in which the
  horizontal axis is the calculated $\rho$ in Step 1 and
  vertical axis is the calculated $\delta$ in Step 2. Points
  with both higher $\rho$ and higher $\delta$ are clustering
  centers. We can also plot $\gamma(=\rho\delta)$ in
  decreasing order to find out the clustering centers.
\item{Step 4.} Order all the roads' density in decreasing
  order. Then screen all the roads from the highest density
  to the lowest one. If a road is not a cluster center, then
  it belongs to the {cluster of} nearest road with higher
  density.
\item{Step 5.} If the number of clusters is bigger than 1,
  then it's necessary to identify whether a road is a core
  road or a halo. Decide the border region for a fixed
  cluster first. Points in this region satisfy that they
  belong to this cluster, but there exist points belonging
  to other clusters within $d_c$. Then calculate the average
  local density based on the border region to identify a
  cluster core or halo. Outliers are often in halos.
\end{itemize}

We do as the steps {described} above to cluster for the
roads with velocity before and after denoising. From the
decision graph, the first row in Figure \ref{Fg:DecGraph},
and the $\gamma$ variance in Figure \ref{Fg:Gamma}, we
should choose 3 clusters for the 102 roads both before and
after denoising. In order to observe the performance of
classification obviously, we calculate the 2 dimensional
non-local multidimensional scaling matrix just
as~\cite{Rodriguez2014Machine} does. This matrix with 2
columns is an approximation of the original distance matrix
whose diagonal elements are all 0 and off-diagonal element
in location $(i,j)$ represents the distance of road {$i$}
and {$j$}. The distance here refers to the $l^2$ norm of
their velocity vectors' difference. With the help of the
approximated scaling matrix, we present the clustering
result in the 2D non-local multidimensional scaling figure,
the second row of Figure \ref{Fg:DecGraph}, where the X axis
and Y axis represent the first and second column of this
matrix correspondingly.
\begin{figure}[ht]
	\centering
		\includegraphics[height=4in,width=2.8in]{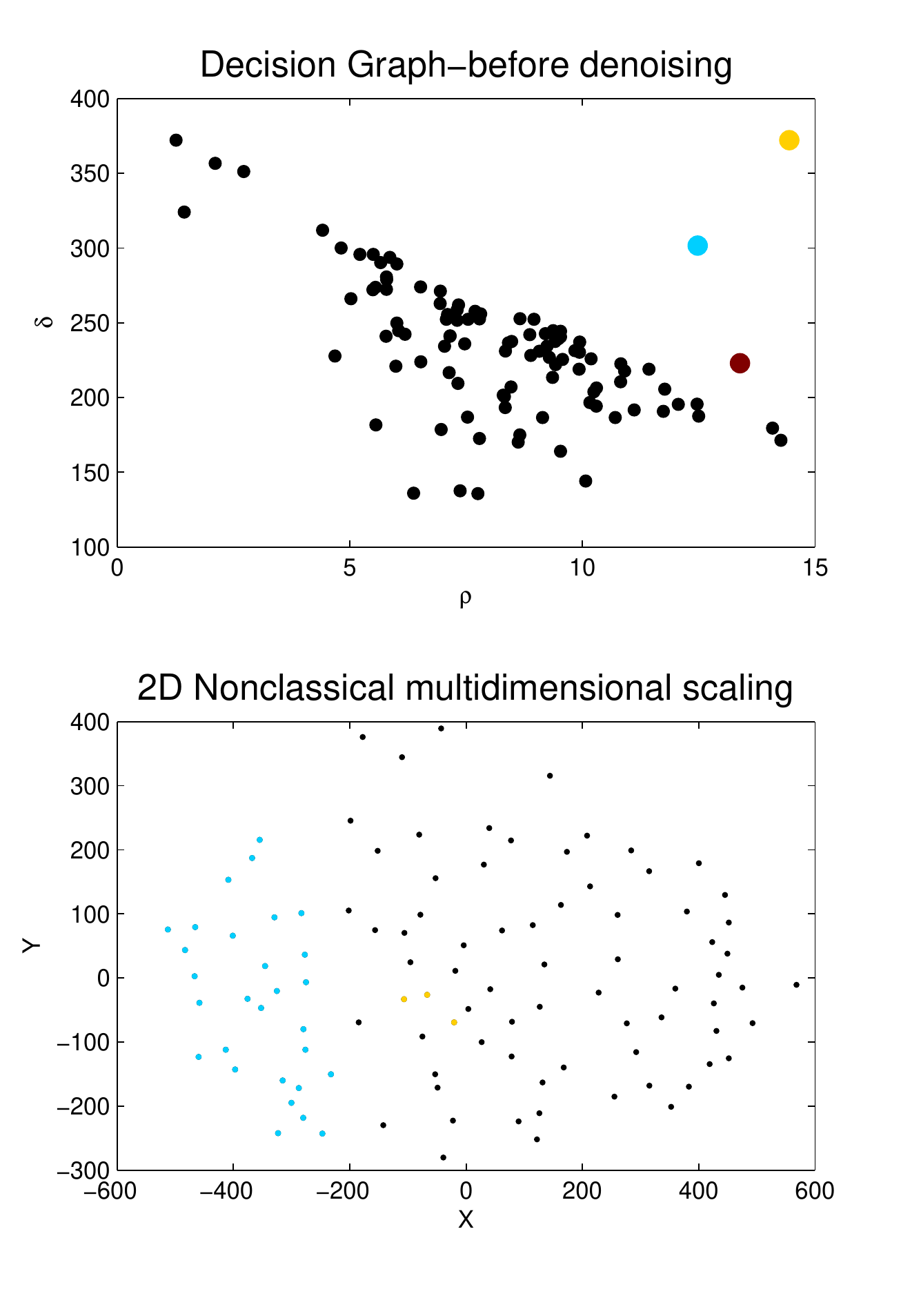}
		\includegraphics[height=4in,width=2.8in]{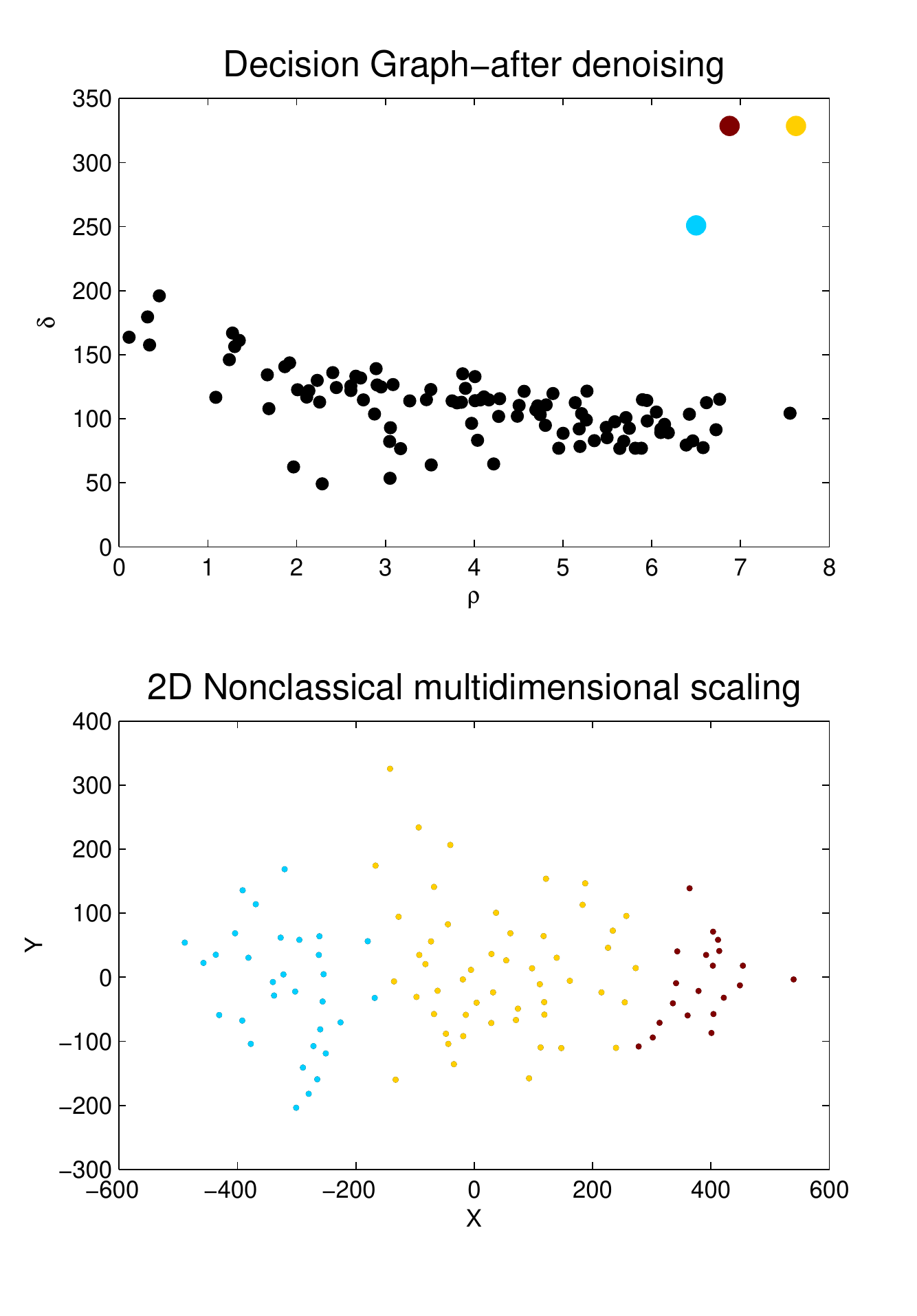}
		\caption{Decision graph and clustering result for the 102 roads. 
		}
		\label{Fg:DecGraph}
\end{figure}
\begin{figure}[!ht]
	\centering
		\includegraphics[height=2.5in,width=2.8in]{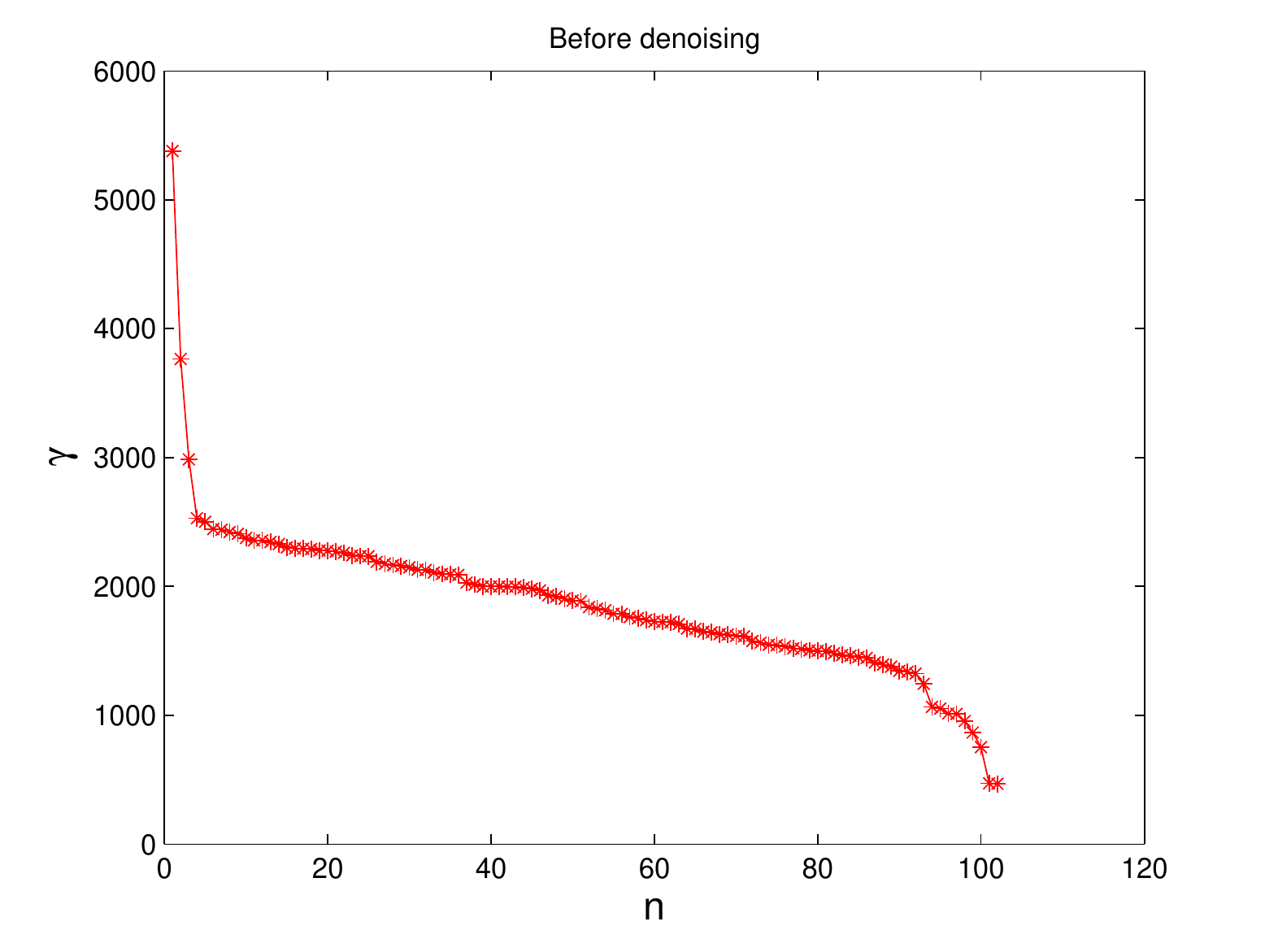}
		\includegraphics[height=2.5in,width=2.8in]{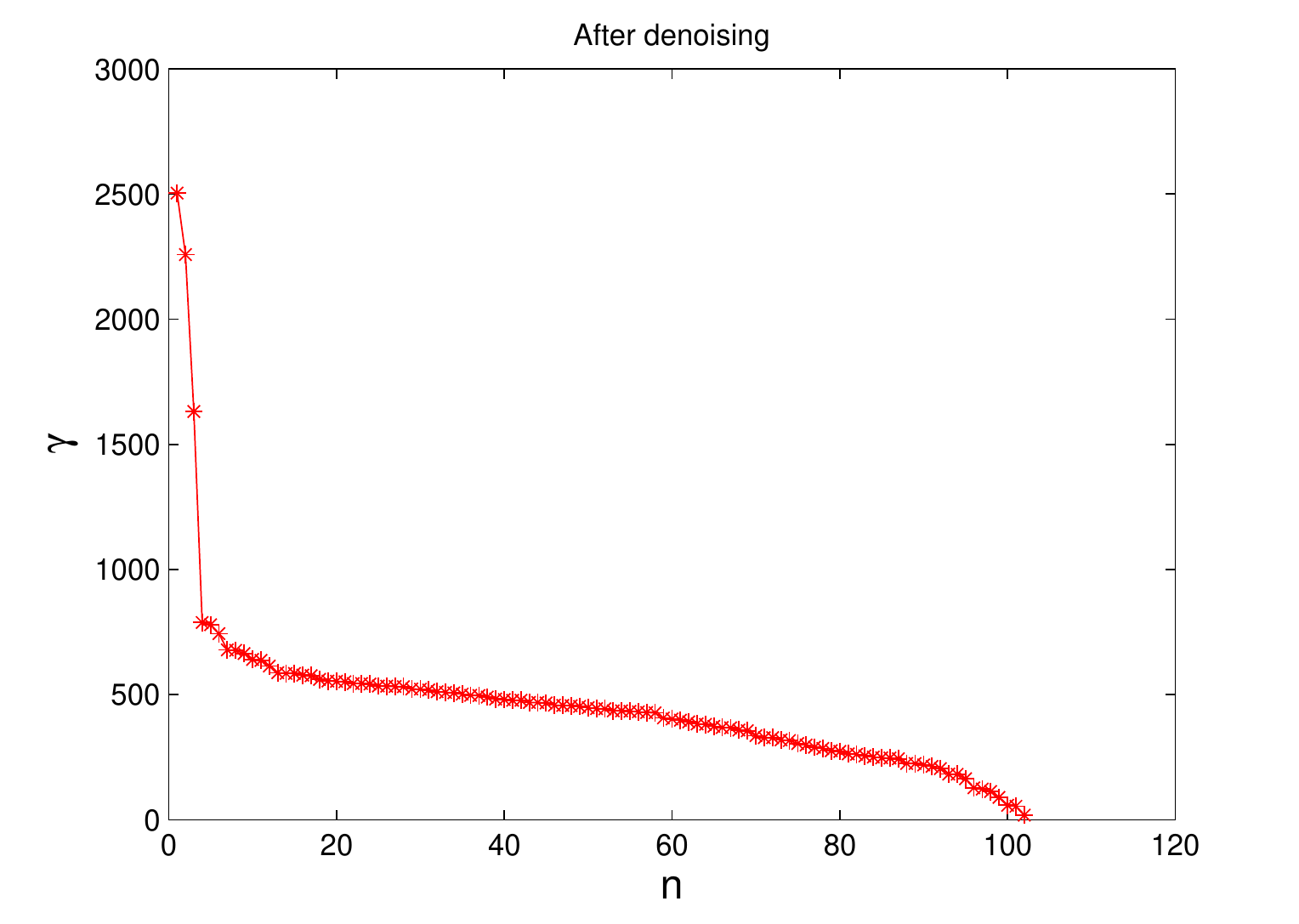}
		\caption{$\gamma$ variance before and after denoising. 
		}
		\label{Fg:Gamma}
\end{figure}

We illustrate the detailed result of clustering in the
following. The 3 cluster centers are the color points in the
first row of Figure \ref{Fg:DecGraph}. The other points are
shown with the color consistent in the second row of Figure
\ref{Fg:DecGraph} and classified as follows:
\begin{itemize} 
\item[(1)]Before denoising, cluster 1 with 29 core elements
  and 0 halos corresponds to the blue points in the 2D
  non-local multidimensional scaling figure, and cluster 2
  with 27 elements only has 3 core elements corresponding to
  the yellow points, and cluster 3 with 46 elements has no
  core elements. In the figure before denoising, if a point
  is a halo, then its color is black in the 2D non-local
  multidimensional scaling figure.
\item[(2)]After denoising, the three clusters all have no
  halos. Cluster 1 has 31 core points, 29 of which are
  exactly the ones that in cluster 1 before denoising, and
  cluster 3 has 20 core points. {Roads belong to cluster
    1,2,3 are represented by} the blue, yellow and purple
  points respectively in the 2D non-local multidimensional
  scaling figure. As there are no halos after denoising, so
  no black points appear this time.
\end{itemize}

Now we give some further explanations on the clustering
result.  {Note that roads in Beijing are divided into 6
  types, which are urban express highway, freeway, national
  highway, provincial highway, prefectural highway and rural
  street, the corresponding type value is from 1 to 6
  respectively.}  In order to illustrate more clearly, we
list some related {quantities} after BV denoising in Table
\ref{tb:3CluPro}.
\begin{table}[!ht]
	\caption{Some Basic Properties of the 3 Clusters.}
	\label{tb:3CluPro}
	\centering
	\begin{tabular}{lllll}
		\hline 
		cluster & average velocity & average $\TV$ & range & range \\
		number & of cluster center &  of cluster center
		 &  of velocity &  of $\TV$\\
		\cline{1-5}
		1& 54.052 & 123.414& [46.165,65.398]& [56.544,259.974]\\
		\cline{1-5}
		2& 32.321 & 104.428& [21.967,48.021]& [95.165,482.275]\\
		\cline{1-5}
		3& 14.965 & 210.578& [6.396,20.641] & [118.177,239.608]\\
		\cline{1-5}
	\end{tabular}
\end{table}
If a negative and positive linear transformation are made
for average velocity and total variation {respectively},
then the trend of variables obtained is consistent with that
of X and Y in the last one of Figure
\ref{Fg:DecGraph}. Therefore, cluster 1 corresponds to roads
with high average velocity and relatively low total
variation. Namely, they should be matched with smaller road
type value, meaning that most of them are urban express
highway and freeway. In fact, there are 9 roads whose type
value is 1, and 17 ways whose value is 2 in cluster
1. Cluster 3 contain roads with low average speed and bigger
total variation, corresponding to prefectural highway or
rural Street. It has been checked that all the roads' type
value are 6 in cluster 3 after denoising. And cluster 2 has
the medial value in terms of average speed and total
variation, indicating it can contain more types of roads,
which is consistent with the fact there are 8 freeways, 15
provincial highways and 28 rural streets.

\begin{figure*}[ht]
	\centering
	\includegraphics[height=2.5in,width=2.8in]{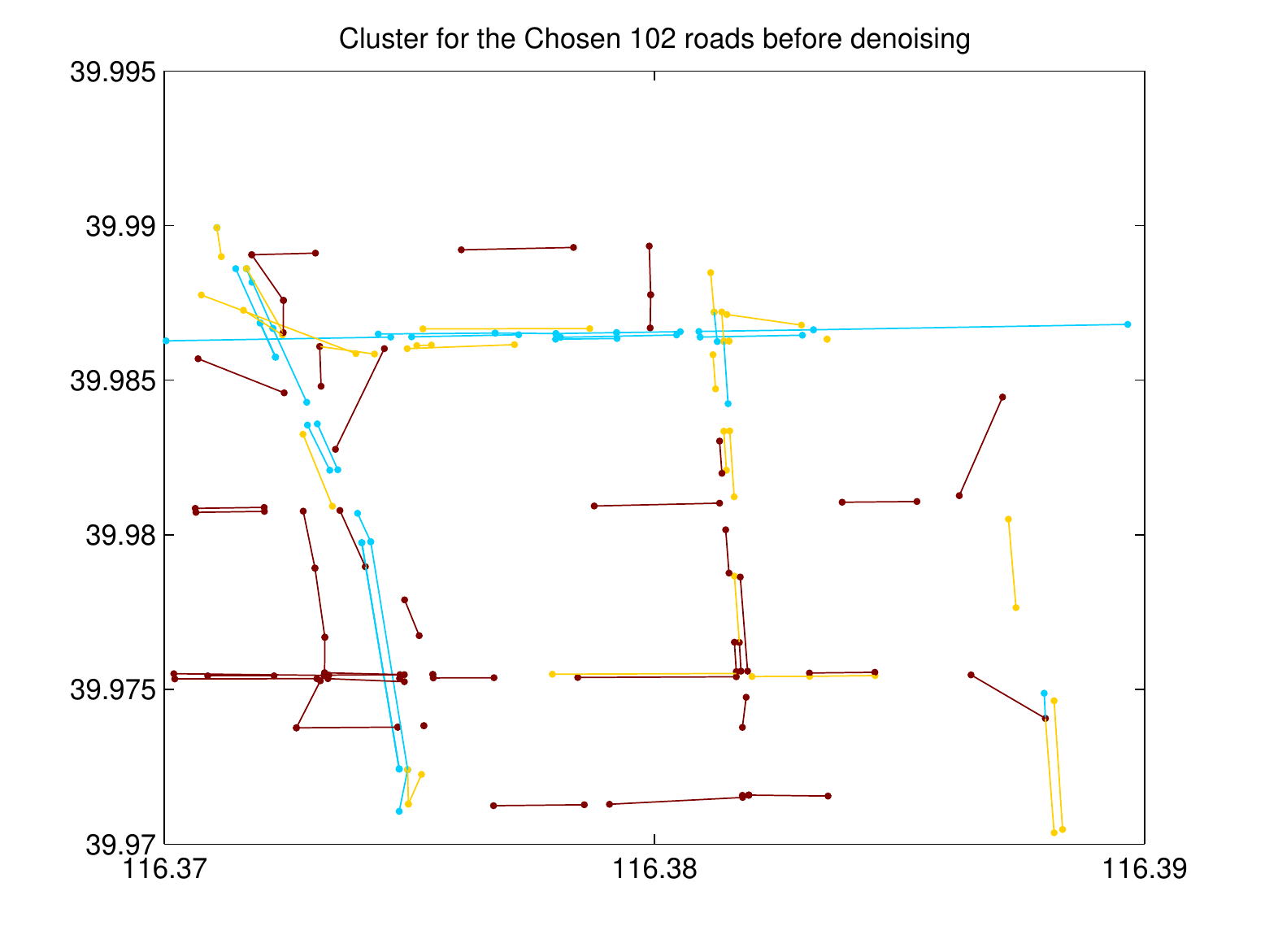}
	\includegraphics[height=2.5in,width=2.8in]{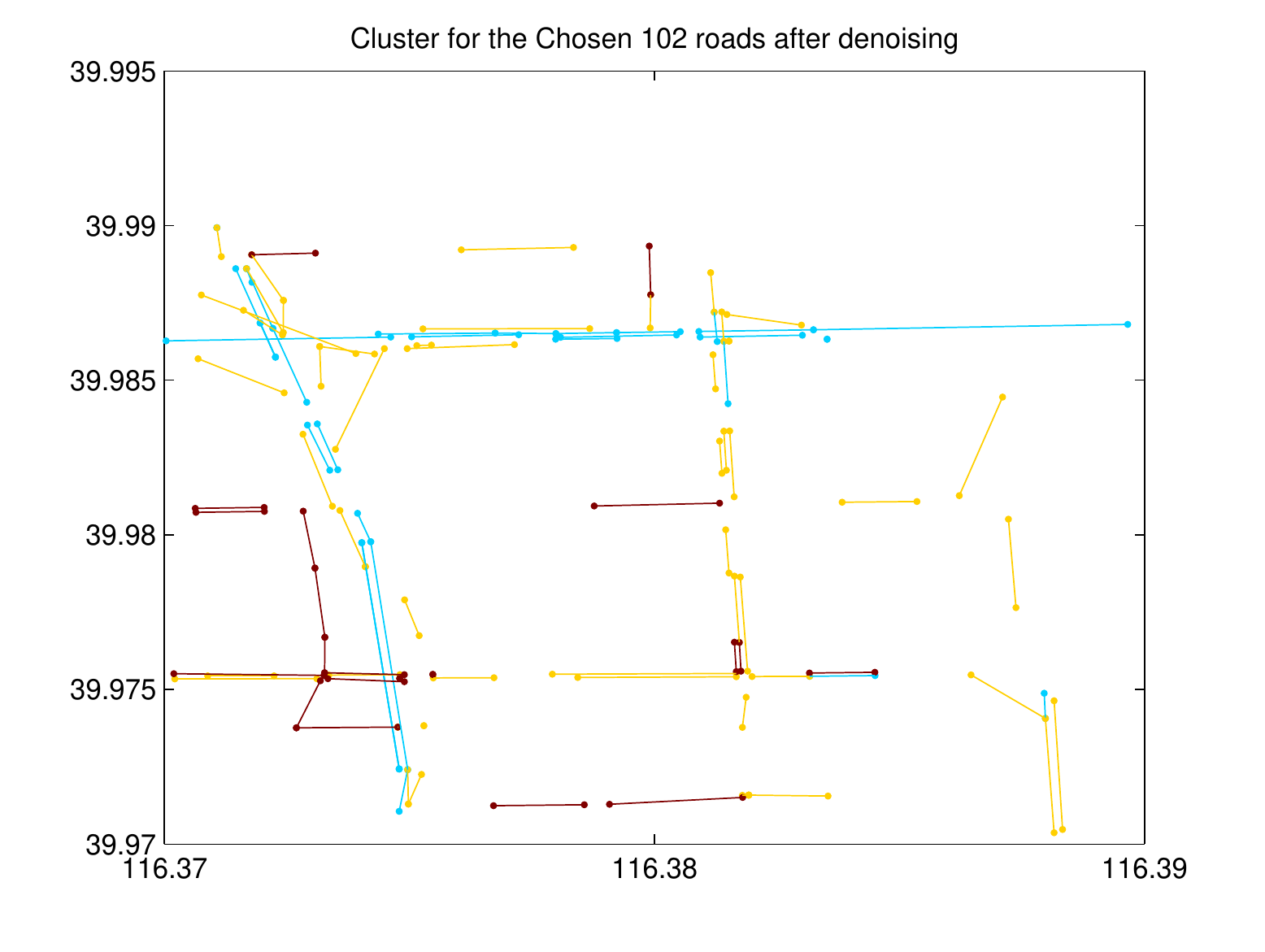} 
	\caption{The clustering result matched with the original 102 roads. 
		The horizontal axis represents longitude, and vertical axis represents latitude.}
	\label{Fg:OriRoad}
\end{figure*}

Finally, in order to observe the clustering result more
directly, we match the 102 points above with the original
102 roads in {corresponding colors} and show them in Figure
\ref{Fg:OriRoad}, where the horizontal axis represents
longitude, and vertical axis means latitude. {From Figure
  \ref{Fg:OriRoad}, we can see the clustering result
  improved greatly after denoising.} There are 28 roads
belonging to different clusters before and after denoising,
and most of them are rural streets. This is consistent with
our common sense.  By BV denoising method, these roads are
discovered and the noise part is eliminated.

{From the clustering results, we see} halos disappear with
the BV denoising method. Because outliers often lie in
halos, so we can say we obtain the clean data by denoising
on the original velocity. This shows our denoising method
works well.

\section{Conclusions}
In this paper, we {applied} the BV denoising method for
urban traffic analysis {by proposing } easy-to-implement
noise strength estimation algorithms.  By applying the BV
denoising method, we have showed the road velocity
prediction accuracy is improved for real urban traffic data
from Beijing Taxi GPS system.  We have also verified that
the road clustering based on velocity profile are improved
significantly by applying the BV denoising method with best
noise-strength estimate.

In this work, we considered only the temporal
characteristics of road velocity profiles. In a follow-up
study, we will consider both temporal and spatial
characteristics, which will be more mathematically involved
and {is expected to} give better results.

\appendix
\renewcommand{\thesection}{A\alph{section}}

\section*{Appendix: Proof of Theorem 2.1.}

In order to prove Theorem \ref{thm: Thm1}, we give a lemma first.
\begin{lemma}\label{lemma: Lem1}
  Let $u_0(x_i)=u_i+\xi_i, i\in[N]$, where $u_0(x_i)$,
  $u_i$, $\xi_i$ represent the observed, denoised data and
  noise for time slice $i$ separately.
  $u_i^{(j)}, i\in[N_{j+1}], j=0,1,2$ are defined as in
  Theorem \ref{thm: Thm1}, and
  $v_i^{j}, i\in[N_{j+1}], j=0,1,2$ are defined as in
  \eqref{eq:vdef}.  Then the following relations {hold}:
	\begin{align*}
	\sum_{i=1}^{N_2}(u_i^{(1)}-v^1_i)^2
	&=\frac{1}{4}\sum_{i=1}^{N_1}(u_i-u_0(x_i))^2
	+\frac{1}{2}\sum_{i=1}^{N_2}(u_{2i-1}-u_0(x_{2i-1}))(u_{2i}-u_0(x_{2i})),\\
	\sum_{i=1}^{N_3}(u_i^{(2)}-v^2_i)^2
	&=\frac{1}{4}\sum_{i=1}^{N_2}(u_i^{(1)}-v^1_i)^2
	+\frac{1}{2}\sum_{i=1}^{N_3}(u_{2i-1}^{(1)}-v^1_{2i-1})(u^{(1)}_{2i}-v^1_{2i}).
	\end{align*}
\end{lemma}
\begin{proof}[Proof of Lemma \ref{lemma: Lem1}]
	1) For the first equation, let $h_2=2h$, we have
	\begin{align*}
	\sum_{i=1}^{N_2}(u_i^{(1)}-v^1_i)^2\cdot h_2 &=\sum_{i=1}^{N_2}(\frac{u_{2i}+u_{2i-1}}{2}-\frac{u_0(x_{2i})+u_0(x_{2i-1})}{2})^2\cdot 2h\\
	&=\sum_{i=1}^{N_2}(\frac{u_{2i}-u_0(x_{2i})}{2}+\frac{u_{2i-1}-u_0(x_{2i-1})}{2})^2\cdot 2h\\
	&=\frac{h}{2}\sum_{i=1}^{N_1}(u_i-u_0(x_i))^2
	+h\sum_{i=1}^{N_2}(u_{2i}-u_0(x_{2i}))(u_{2i-1}-u_0(x_{2i-1}))
	\end{align*}
	Rearrange the equation above and then the formula for $\sum_{i=1}^{N_2}(u_i^{(1)}-v^1_i)^2$ is obtained.
	
	2) For the second equation, let $h_3=4h$, the proof can be done with similar procedure. \qedsymbol
\end{proof}
	
\begin{proof}[Proof of Theorem  \ref{thm: Thm1}]
	We prove this theorem in three steps.\\
	(1) By definition of $V_1$, we get
	\begin{equation*}
	V_1 
	=\frac{1}{h}(2\sum_{i=1}^{N_1}(u_0(x_i))^2-(u_0(x_1))^2-(u_0(x_{N_1}))^2
	-2\sum_{i=1}^{N_1-1}u_0(x_i)u_0(x_{i+1})).
	\end{equation*}
	By using the representation of $\sum_{i=1}^{N_1}(u_0(x_i))^2$ from equation \eqref{eq:Def-sigma}, we get
	\begin{align*}
	V_1=\frac{4\sigma^2}{h^2}+\frac{1}{h}(4\sum_{i=1}^{N_1}u_iu_0(x_i)-2\sum_{i=1}^{N_1}(u_i)^2
	-(u_0(x_1))^2
	-(u_0(x_{N_1}))^2-2\sum_{i=1}^{N_1-1}u_0(x_i)u_0(x_{i+1})).
	\end{align*}
	Similarly, combine the definition of $V_2$, $V_3$ and Lemma
	\ref{lemma: Lem1}, we obtain the representation of $V_2$ and $V_3$
	\begin{align*}
	V_2 &=\frac{\sigma^2}{2h^2}+\frac{1}{2h}\Big(\sum_{i=1}^{N_2}(u_{2i-1}-u_0(x_{2i-1}))(u_{2i}-u_0(x_{2i}))\\
	& \quad+4\sum_{i=1}^{N_2}u_i^{(1)} v^1_i
	-2\sum_{i=1}^{N_2}(u_i^{(1)})^2
	-(v^1_1)^2
	-(v^1_{N_2})^2-2\sum_{i=1}^{N_2-1}v^1_i v^1_{i+1}\Big),
	\\
	V_3&=\frac{\sigma^2}{16h^2}
	+\frac{1}{4h}\Big(\frac{1}{4}\sum_{i=1}^{N_2}(u_{2i-1}-u_0(x_{2i-1} ) ) (u_{2i}-u_0(x_{2i})) 
	+\sum_{i=1}^{N_3}(u^{(1)}_{2i-1}-v^1_{2i-1})(u_{2i}^{(1)}-v^1_{2i})\\
	&\quad +4\sum_{i=1}^{N_3}u_i^{(2)} v^2_i
	-2\sum_{i=1}^{N_3}(u_i^{(2)})^2-(v^2_1)^2
	-(v^2_{N_3})^2-2\sum_{i=1}^{N_3-1} v^2_i v^2_{i+1}\Big).
	\end{align*} 
	(2) Now, we are ready to apply the assumptions in the theorem.\\
	After regrouping the order of the representation of $V_1, V_2, V_3$ and replacing $u_0(x_i)$ with $u_i+\xi_i$, we obtain
	\begin{align*}
	V_1
    &=\frac{2}{h}\sum_{i=1}^{N_1-1}( (u_{i+1}-u_i)(\xi_{i+1}-\xi_{i})-\xi_i\xi_{i+1} )
	+\frac{4\sigma^2}{h^2}-\frac{\xi_1^2}{h}-\frac{\xi_{N_1}^2}{h}+\sum_{i=1}^{N_1-1}(u_i-u_{i+1})^2\frac{1}{h},\\
	V_2&=\frac{\sigma^2}{2h^2}+\frac{1}{2h}(\sum_{i=1}^{N_2}\xi_{2i-1}\xi_{2i}-(\frac{\xi_1+\xi_2}{2})^2-(\frac{\xi_{N_1-1}+\xi_{N_1}}{2})^2\\
	&\quad +\sum_{i=1}^{N_2-1}((u_{i+1}^{(1)}-u_i^{(1)})(\xi_{2i+1}+\xi_{2i+2}-\xi_{2i-1}-\xi_{2i}) -\frac{(\xi_{2i-1}+\xi_{2i})(\xi_{2i+1}+\xi_{2i+2})}{2})
	+\sum_{i=1}^{N_2-1}(u_{i+1}^{(1)}-u_i^{(1)})^2),\\
	V_3&=\frac{\sigma^2}{16h^2}+\frac{1}{4h}(\frac{1}{4}\sum_{i=1}^{N_2}\xi_{2i-1}\xi_{2i}+
	\sum_{i=1}^{N_3}\frac{(\xi_{4i-3}+\xi_{4i-2})(\xi_{4i-1}+\xi_{4i})}{4} -\frac{(\sum_{i=1}^4 \xi_i)^2+(\sum_{j=0}^3 \xi_{N_1-j})^2}{16}\\
	&\quad +\sum_{i=1}^{N_3-1}(u_i^{(2)}-u_{i+1}^{(2)})^2
	+\sum_{i=1}^{N_3-1}(\frac{(u_{i+1}^2-u_i^{(2)})(\sum_{j=1}^4(\xi_{4i+j}-\xi_{4i+1-j}))}{2}
	-\frac{(\sum_{j=0}^3\xi_{4i-j})(\sum_{j=1}^4\xi_{4i+j})}{8})).
	\end{align*}
	Take expectation on the equations above and combine the  assumptions, then
	\begin{align*}
	\E(V_1) &=(4-\frac{4}{N})\frac{\sigma^2}{h^2}+\sum_{i=1}^{N_1-1}(u_{i+1}-u_i)^2\cdot\frac{1}{h},\\
	\E(V_2) &=(\frac{1}{2}-\frac{1}{N})\frac{\sigma^2}{h^2}+\sum_{i=1}^{N_2-1}(u_{i+1}^{(1)}-u_i^{(1)})^2\cdot\frac{1}{2h},\\
	\E(V_3) & =(\frac{1}{16}-\frac{1}{4N})\frac{\sigma^2}{h^2}+\sum_{i=1}^{N_3-1}(u_{i+1}^{(2)}-u_i^{(2)})^2\cdot\frac{1}{4h}.
	\end{align*}
	
	From what have done above, we can estimate $\sigma$ by fitting a line with 3 points:
	$(4-\frac{4}{N},V_1),(\frac{1}{2}-\frac{1}{N},V_2),(\frac{1}{16}-\frac{1}{4N},V_3)$. The slope of the line is $\frac{\sigma^2}{h^2}$, then $\sigma$ is estimated.\\
	(3) Estimate the value of $\sigma$.
    Set 
	\begin{equation*}
	\begin{aligned}
	\label{eq:DefiForXiYi}
	X_1=4-\frac{4}{N}, \quad
	X_2=\frac{1}{2}-\frac{1}{N}, \quad
	X_3=\frac{1}{16}-\frac{1}{4N},
	\end{aligned}
	\end{equation*}
	and assume $(X_1,V_1),(X_2,V_2),(X_3,V_3)$ satisfy
	\begin{equation*}
	V_i=AX_i+b+\mu_i,\quad i=1,2,3,
	\end{equation*}
	where $\mu_i$ {stand for some small} random disturbance variables.\\
	Then we want to fit a line with these 3 points, which is 
	\begin{equation*}
	V=\hat{A}X+\hat{b}.
	\end{equation*}
	With the help of least square method, we obtain the representation of $\hat{A}$ as follows:
	\begin{equation}\label{eq:A hat}
	\hat{A}=\frac{3\sum_{i=1}^3 X_iV_i-\sum_{i=1}^3X_i\sum_{i=1}^3V_i}
	{3\sum_{i=1}^3X_i^2-(\sum_{i=1}^3X_i)^2}.
	\end{equation}
	So we estimate $\sigma^2$ by
	\begin{equation*}
	\hat{\sigma}^2=h^2\frac{(\frac{119}{16}-\frac{27}{4N})V_1+(\frac{9}{4N}-\frac{49}{16})V_2+(\frac{9}{2N}-\frac{35}{8})V_3}{\frac{3577}{128}+\frac{189}{8N^2}-\frac{819}{16N}}.
    \end{equation*}
	Combine the assumption $V_i=AX_i+b+\mu_i,i=1,2,3,$ and \eqref{eq:A hat}, we get
	\begin{equation}\label{eq:Ahat--Mu}
	\hat{A}=A+\frac{3\sum_{i=1}^3X_i\mu_i-\sum_{i=1}^3X_i\sum_{i=1}^3\mu_i}
	{3\sum_{i=1}^3X_i^2-(\sum_{i=1}^3X_i)^2}.
	\end{equation}
	Thus,
	\begin{equation}\label{eq:ExpecAhat}
	\E(\hat{A})=A+\frac{3\sum_{i=1}^3X_i\E(\mu_i)-\sum_{i=1}^3X_i\sum_{i=1}^3\E(\mu_i)}
	{3\sum_{i=1}^3X_i^2-(\sum_{i=1}^3X_i)^2}.
	\end{equation}
	Notice that by the construct of $V_i$ and the linear assumption for $V_i,i=1,2,3$, the following relations are established:
	\begin{align*}
	& b+\mu_1=\sum_{i=1}^{N_1-1}|u_{i+1}-u_i|^2\cdot\frac{1}{h}=V_1^{c},\\
	& b+\mu_2=\sum_{i=1}^{N_2-1}|u^{(1)}_{i+1}-u^{(1)}_i|^2\cdot\frac{1}{2h}=V_2^{c},\\
	& b+\mu_3=\sum_{i=1}^{N_3-1}|u^{(2)}_{i+1}-u^{(2)}_i|^2\cdot\frac{1}{4h}=V_3^{c}.
	\end{align*} 
	Substitute $\E(\mu_i)=V_i^{c}-b$ into \eqref{eq:ExpecAhat}, we derive
	\begin{equation*}
	\E(\hat{A})=A+\frac{(\frac{49}{16}-\frac{9}{4N})(V_1^c-V_2^c)+(\frac{35}{8}-\frac{9}{2N})(V_1^c-V_3^c)}{\frac{3577}{128}+\frac{189}{8N^2}-\frac{819}{16N}}.
	\end{equation*} 
	Denote 
	\begin{equation}
	Bias(N)=h^2\frac{(\frac{49}{16}-\frac{9}{4N})(V_1^c-V_2^c)+(\frac{35}{8}-\frac{9}{2N})(V_1^c-V_3^c)}{\frac{3577}{128}+\frac{189}{8N^2}-\frac{819}{16N}}.
	\end{equation}
	Because $\hat{A}=\frac{\hat{\sigma}^2}{h^2}$, $A=\frac{\sigma^2}{h^2}$,
	we get equation \eqref{eq:varEstimate1}. \qedsymbol
\end{proof}

\section*{Acknowledgments}
The authors would like to thank Beijing Transportation
Information Center for providing us some valuable Taxi
GPS-based data for research.  The authors would also like to
thank Prof. Tiejun Li, Prof. Weinan E and Dr. Yucheng Hu for
helpful discussions.  The work is supported by China
National Program on Key Basic Research Project 2015CB856003.










\begin{thebibliography}{99}

\bibitem{Liu2011Discovering} W.~Liu, Y.~Zheng, S.~Chawla,
  J.~Yuan, and X.~Xing, {\em Discovering spatio-temporal
    causal interactions in traffic data streams}, in {\em
    ACM SIGKDD International Conference on Knowledge
    Discovery and Data Mining}, pp. 1010--1018 (2011).

\bibitem{Wang2013Visual} Z.~Wang, M.~Lu, X.~Yuan, J.~Zhang,
  and H.~V.~D. Wetering, {\em Visual traffic jam analysis
    based on trajectory data}, {\em IEEE Transactions on
    Visualization $\&$ Computer Graphics}, \textbf{19}(12),
  pp. 2159--2168 (2013).

\bibitem{Yuan2010T} J.~Yuan, Y.~Zheng, C.~Zhang, W.~Xie,
  X.~Xie, G.~Sun, and Y.~Huang, {\em T-drive: driving
    directions based on taxi trajectories}, in {\em
    {SIGSPATIAL} International Conference on Advances in
    Geographic Information Systems}, pp. 99--108 (2010).

\bibitem{Kim2015Spatial} J.~Kim and H.~S. Mahmassani, {\em
    Spatial and temporal characterization of travel Patterns
    in a Traffic Network Using Vehicle Trajectories}, {\em
    Transportation Research Part C: Emerging Technologies},
  \textbf{59}, pp.  375--390 (2015).

\bibitem{Palma2008A} A.~T. Palma, V.~Bogorny, B.~Kuijpers,
  and L.~O. Alvares, {\em A clustering-based approach for
    discovering interesting places in trajectories}, in {\em
    ACM Symposium on Applied Computing}, pp. 863--868
  (2008).

\bibitem{Zhan2017Citywide} X.~Zhan, Y.~Zheng, X.~Yi, and
  S.~Ukkusuri, {\em Citywide Traffic Volume Estimation Using
    Trajectory Data}, {\em IEEE Transactions on Knowledge
    $\&$ Data Engineering}, \textbf{29}(2), pp. 272--285
  (2017).

\bibitem{Br2015Artificial} A.~D. Br\'ebisson, E.~Simon,
  A.~Auvolat, P.~Vincent, and Y.~Bengio, {\em Artificial
    neural networks applied to taxi destination prediction},
  in {\em Proceedings of the 2015 International Conference
    on ECML PKDD Discovery Challenge} \textbf{1526}
  pp. 40--51 (2015).

\bibitem{Sadek2004Multi} N.~Sadek and A.~Khotanzad, {\em
    Multi-scale high-speed network traffic prediction using
    $k$-factor {Gegenbauer ARMA} model}, in {\em IEEE
    International Conference on Communications}, \textbf{4}
  pp. 2148--2152 (2004).

\bibitem{Wanli2011Real} W.~Min and L.~Wynter, {\em Real-time
    road traffic prediction with spatio-temporal
    correlations}, {\em Transportation Research Part C:
    Emerging Technologies}, \textbf{19}(4), pp. 606--616
  (2011).

\bibitem{Zhu2013A} Y.~Zhu, Z.~Li, H.~Zhu, M.~Li, and
  Q.~Zhang, {\em A compressive sensing approach to urban
    traffic estimation with probe vehicles}, {\em IEEE
    Transactions on Mobile Computing}, \textbf{12}(11),
  pp. 2289--2302 (2013).

\bibitem{Xiao2003Fuzzy} H.~Xiao, H.~Sun, and B.~Ran, {\em
    Fuzzy-neural network traffic prediction framework with
    wavelet decomposition}, {\em Transportation Research
    Record Journal of the Transportation Research Board},
  \textbf{1836}(1), pp.16--20 (2003).

\bibitem{Zheng2016Traffic} Z.~Zheng and D.~Su, {\em Traffic
    state estimation through compressed sensing and {Markov}
    random field}, \emph{Transportation Research Part B
    Methodological}, \textbf{91} pp. 525--554 (2016).

\bibitem{Yarotsky2017Error} D.~Yarotsky, {\em Error bounds
    for approximations with deep {ReLU} networks},
  \emph{Neural Networks}, \textbf{94}, pp. 103--114 (2017).

\bibitem{Liang2017Why} S.~Liang and R.~Srikant, {\em Why
    deep neural networks for function approximation?},
  \emph{arXiv:1610.04161}, 2017.

\bibitem{tang2010deep} Y.~Tang and C.~Eliasmith, {\em Deep
    networks for robust visual recognition}, in
  \emph{Proceedings of the 27th International Conference on
    Machine Learning}, pp. 1055--1062 (2010).

\bibitem{nowlan1992simplifying} S.~J. Nowlan and
  G.~E. Hinton, {\em Simplifying neural networks by soft
    weight-sharing}, \emph{Neural computation},
  \textbf{4}(4), pp. 473--493 (1992).

\bibitem{vincent2008extracting} P.~Vincent, H.~Larochelle,
  Y.~Bengio, and P.-A. Manzagol, {\em Extracting and
    composing robust features with denoising autoencoders},
  in \emph{Proceedings of the 25th International Conference
    on Machine Learning, ACM}, pp. 1096--1103 (2008).

\bibitem{vincent2010stacked} P.~Vincent, H.~Larochelle,
  I.~Lajoie, Y.~Bengio, and P.-A. Manzagol, {\em Stacked
    denoising autoencoders: Learning useful representations
    in a deep network with a local denoising criterion},
  \emph{Journal of Machine Learning Research}, \textbf{11},
  pp. 3371--3408 (2010).

\bibitem{maaten2013learning} L.~Maaten, M.~Chen, S.~Tyree,
  and K.~Weinberger, {\em Learning with marginalized
    corrupted features}, in \emph{International Conference
    on Machine Learning}, pp. 410--418 (2013).

\bibitem{srivastava2014dropout} N.~Srivastava, G.~Hinton,
  A.~Krizhevsky, I.~Sutskever, and R.~Salakhutdinov, {\em
    Dropout: a simple way to prevent neural networks from
    overfitting}, \emph{The Journal of Machine Learning
    Research}, \textbf{15}(1), pp.  1929--1958 (2014).

\bibitem{Rudin1992Nonlinear} L.~I. Rudin, S.~Osher, and
  E.~Fatemi, {\em Nonlinear total variation based noise
    removal algorithms}, \emph{Physica D: Nonlinear
    Phenomena}, \textbf{60}(1), pp. 259--268 (1992).

\bibitem{You.Kaveh2000} Y.-L. You and M.~Kaveh, {\em
    Fourth-order partial differential equations for noise
    removal}, \emph{IEEE Trans. Image Process.},
  \textbf{9}(10), pp. 1723--1730, (2000).

\bibitem{M2005Fast} M.~Mahmoudi and G.~Sapiro, {\em Fast
    image and video denoising via nonlocal means of similar
    neighborhoods}, \emph{IEEE {Signal Processing Letters}},
  \textbf{12}(12), pp. 839--842 (2005).

\bibitem{Bao2005Image} P.~Bao and X.~Ma, {\em Image adaptive
    watermarking using wavelet domain singular value
    decomposition}, \emph{IEEE Transactions on Circuits and
    Systems for Video Technology}, \textbf{15}(1),
  pp. 96--102 (2005).

\bibitem{Baraniuk2007Compressive} R.~G. Baraniuk, {\em
    Compressive sensing}, \emph{IEEE Signal Processing
    Magazine}, \textbf{24}(4), pp. 118--121 (2007).

\bibitem{Yin.etal2008} W.~Yin, S.~Osher, D.~Goldfarb, and
  J.~Darbon, {\em Bregman iterative algorithms for
    $\ell_1$-minimization with applications to compressed
    sensing}, \emph{SIAM J. Imaging Sci.}, \textbf{1}(1),
  pp. 143--168 (2008).

\bibitem{Dabov2008Video} K.~Dabov, A.~Foi, and
  K.~Egiazarian, {\em Video denoising by sparse 3d
    transform-domain collaborative filtering}, in
  \emph{Signal Processing Conference, 2007 European},
  pp. 145--149 (2008).

\bibitem{Cai.etal2009} J.~Cai, S.~Osher, and Z.~Shen, {\em
    Split bregman methods and frame based image
    restoration}, \emph{Multiscale Model. Simul.},
  \textbf{8}(2), pp. 337--369 (2009).

\bibitem{Wu.Tai2010a} C.~Wu and X.~Tai, {\em Augmented
    {Lagrangian} method, dual methods, and split {Bregman}
    iteration for {ROF}, vectorial {TV}, and high order
    models}, \emph{SIAM J. Imaging Sci.}, \textbf{3}(3),
  pp. 300--339 (2010).

\bibitem{Liu2010A} C.~Liu and W.~T. Freeman, \emph{A
High-Quality Video Denoising Algorithm Based on Reliable
Motion Estimation}, in {\em European Conference on Computer Vision} (2010).

\bibitem{Ji2010Robust} H.~Ji, C.~Liu, Z.~Shen, and Y.~Xu,
  {\em Robust video denoising using low rank matrix
    completion}, in \emph{Computer Vision and Pattern
    Recognition}, pp. 1791--1798 (2010).

\bibitem{Pang.etal2011} Z.-F. Pang, L.-L. Wang, and
  Y.-F. Yang, {\em Fast algorithms for the anisotropic {LLT}
    model in image denoising}, \emph{East Asian
    J. Appl. Math.}, \textbf{1}(3), pp. 264--283 (2011).

\bibitem{Shi.etal2012} B.~Shi, Z.-F. Pang, and Y.-F. Yang,
  {\em A projection method based on the splitting {Bregman}
    iteration for the image denoising}, \emph{J. Appl. Math.
    Comput.}, \textbf{39}(1), pp. 533--550 (2012).

\bibitem{Zhang.etal2017b} X.~Zhang, Y.~Shi, Z.-F. Pang, and
  Y.~Zhu, {\em Fast algorithm for image denoising with
    different boundary conditions}, \emph{Journal of the
    Franklin Institute}, \textbf{354}(11), pp. 4595--4614
  (2017).

\bibitem{Yang.etal2018b} W.~Yang, Z.~Huang, and W.~Zhu, {\em
    An efficient tailored finite point method for {Rician}
    denoising and deblurring}, \emph{Commun Comput Phys},
  \textbf{24} pp. 1169-1195 (2018).

\bibitem{wei2017learning} T.~Wei, L.~Wang, P.~Lin, J.~Chen,
  Y.~Wang, and H.~Zheng, {\em Learning non-negativity
    constrained variation for image denoising and
    deblurring}, \emph{Numerical Mathematics: Theory,
    Methods and Applications}, \textbf{10}(4), pp. 852--871
  (2017).

\bibitem{sciacchitano2017total} F.~Sciacchitano, Y.~Dong,
  and M.~S. Andersen, {\em Total variation based
    parameter-free model for impulse noise removal},
  \emph{Numerical Mathematics: Theory, Methods and
    Applications}, \textbf{10}(1), pp. 186--204 (2017).

\bibitem{jiang2018A} D.~Jiang, X.~Wang, G.~Xu, and J.~Lin,
  {\em A denoising-decomposition model combining {TV}
    minimisation and fractional derivatives}, \emph{East
    Asian Journal on Applied Mathematics}, \textbf{8}(3),
  pp. 447--462 (2018).

\bibitem{Zheng2011Applications} Z.~Zheng, S.~Ahn, D.~Chen,
  and J.~Laval, {\em Applications of wavelet transform for
    analysis of freeway traffic: Bottlenecks, transient
    traffic, and traffic oscillations}, \emph{Transportation
    Research Part B Methodological}, \textbf{45}(2),
  pp. 372--384 (2011).

\bibitem{Xu2015The} D.~W. Xu, H.~H. Dong, H.~J. Li,
  L.~M. Jia, and Y.~J. Feng, {\em The estimation of road
    traffic states based on compressive sensing},
  \emph{Transportmetrica B:Transport Dynamics},
  \textbf{3}(2), pp. 131--152 (2015).

\bibitem{Strong2003Edge} D.~Strong and T.~Chan, {\em
    Edge-preserving and scale-dependent properties of total
    variation regularization}, \emph{Inverse Problems},
  \textbf{19}(6), pp. 165--187 (2003).

\bibitem{Rosen1961The} J.~B. Rosen, {\em The Gradient
    Projection Method for Nonlinear Programming. Part
    I. Linear Constraints}, \emph{Journal of the Society for
    Industrial $\&$ Applied Mathematics}, \textbf{9}(4),
  pp. 514--532 (1961).

\bibitem{Antonin2009An} A.~Chambolle, V.~Caselles,
  M.~Novaga, D.~Cremers, and T.~Pock, {\em An introduction
    to total variation for image analysis}, in {\em Radon
    Series Comp. Appl. Math}, \textbf{9}, pp. 263--340
  (2010).

\bibitem{Wu.etal2011} C.~Wu, J.~Zhang, and X.-C. Tai, {\em
    Augmented {Lagrangian} method for total variation
    restoration with non-quadratic fidelity}, \emph{Inverse
    Probl.  Imaging}, \textbf{5}(1), pp. 237--261 (2011).

\bibitem{Cand2009Exact} E.~J. Cand\`es and B.~Recht, {\em
    Exact Matrix Completion via Convex Optimization},
  \emph{Foundations of Computational Mathematics},
  \textbf{9}, pp. 717 (2009).

\bibitem{Fazel2001A} M.~Fazel, H.~Hindi, and S.~P. Boyd,
  {\em A rank minimization heuristic with application to
    minimum order system approximation}, in
  \emph{Proceedings of the American Control Conference},
  \textbf{6}, pp. 4734--4739 (2001).

\bibitem{Rodriguez2014Machine} A.~Rodriguez and A.~Laio,
  {\em Clustering by fast search and find of density peaks},
  \emph{Science}, \textbf{344}(6191), p. 1492 (2014).

\bibitem{Xiong2006Enhancing} H.~Xiong, G.~Pandey,
  M.~Steinbach, and V.~Kumar, {\em Enhancing data analysis
    with noise removal}, \emph{IEEE Transactions on
    Knowledge $\&$ Data Engineering}, \textbf{18}(3),
  pp. 304--319 (2006).

\bibitem{liu2010understanding} Y.~Liu, Z.~Li, H.~Xiong,
  X.~Gao, and J.~Wu, {\em Understanding of internal
    clustering validation measures}, in \emph{IEEE 10th
    International Conference on Data Mining (ICDM)},
  pp. 911--916 (2010).
\end{thebibliography}


\end{document}